\documentclass{article}
\pdfoutput=1
\usepackage{pgfplots}
\usepackage{pgfplotstable}
\usepgfplotslibrary{groupplots,units}
\usepackage{microtype}
\usepackage{graphicx}
\usepackage{subfigure}
\usepackage{hyperref}
\usepackage{etoolbox}
\makeatletter
\patchcmd\@combinedblfloats{\box\@outputbox}{
\unvbox\@outputbox}{}{\errmessage{\noexpand patch failed}}
\makeatother
\usepackage{booktabs}
\usepackage{inputenc}
\usepackage{amsthm}
\usepackage{amsmath}
\DeclareMathOperator*{\argmax}{arg\,max}
\DeclareMathOperator*{\argmin}{arg\,min}
\usepackage{amsfonts}
\usepackage{bm}

\newtheorem{remark}{Remark}
\newtheorem{theorem}{Theorem}
\newtheorem{lemma}{Lemma}
\newtheorem{corollary}{Corollary}
\newtheorem{assumption}{Assumption}
\newtheorem{definition}{Definition}
\usepackage{tikz}
\usepackage[accepted]{icml2019}
\icmltitlerunning{Stochastic Iterative Hard Thresholding for Graph-structured Sparsity Optimization}
\pgfplotsset{width=7cm,compat=1.8}

\begin{document}
\twocolumn[
\icmltitle{Stochastic Iterative Hard Thresholding for Graph-structured Sparsity  Optimization}

\icmlsetsymbol{equal}{*}

\begin{icmlauthorlist}
\icmlauthor{Baojian Zhou}{cs}
\icmlauthor{Feng Chen}{cs}
\icmlauthor{Yiming Ying}{math}
\end{icmlauthorlist}

\icmlaffiliation{cs}{Department of Computer Science, SUNY at Albany, Albany, NY, USA}
\icmlaffiliation{math}{Department of Mathematics and Statistics, SUNY at Albany, Albany, NY, USA}

\icmlcorrespondingauthor{Baojian Zhou}{bzhou6@albany.edu}

\icmlkeywords{Graph-structured sparsity, Stochastic gradient descent, Hard Thresholding, Non-Convex Optimization}
\vskip 0.3in
]
\printAffiliationsAndNotice{}

%%%%%%%%%%%%%%%%%%%%%%%%%%%%%%%%%% Abstract %%%%%%%%%%%%%%%%%%%%%%%%%%%%%%%%%%%%
\begin{abstract}
Stochastic optimization algorithms update models with cheap per-iteration costs sequentially, which makes them amenable for large-scale data analysis. Such algorithms have been widely studied for structured sparse models where the sparsity information is very specific, e.g., convex sparsity-inducing norms or $\ell^0$-norm. However, these norms cannot be directly applied to the problem of complex (non-convex) graph-structured sparsity models, which have important application in disease outbreak and social networks, etc. In this paper, we propose a stochastic gradient-based method for solving graph-structured sparsity constraint problems, not restricted to the least square loss. We prove that our algorithm enjoys a linear convergence up to a constant error, which is competitive with the counterparts in the batch learning setting. We conduct extensive experiments to show the efficiency and effectiveness of the proposed algorithms. 
\end{abstract}

%%%%%%%%%%%%%%%%%%%%%%%%%%%%%%%% Introduction %%%%%%%%%%%%%%%%%%%%%%%%%%%%%%%%%%
\section{Introduction}
\label{section-introduction}
Structured sparse learning  models  have received increasing attention. They can be formulated as follows
\begin{equation}
    \min_{ \bm{x} \in \mathcal{M}}  F({\bm x}), \quad F({\bm x}) 
    := \frac{1}{n} \sum_{i=1}^n f_i ({\bm x}).
    \label{equ:objective_function}
\end{equation}
Here, each $f_i({\bm x})$ is convex and differentiable and the structured sparsity is reflected by the constraint  set $\mathcal{M}\subseteq \mathbb{R}^p$ on ${\bm x}$. Typically, one can encode the sparsity by introducing sparsity-inducing penalties such as $\ell^1$-norm \cite{tibshirani1996regression,chen2001atomic}, $\ell^1/\ell^q$ mixed norms \cite{turlach2005simultaneous,yuan2006model} and more structured norms built on either disjoint or overlapping groups of variables \cite{bach2012structured,jenatton2011structured,obozinski2016unified,morales2010family}.  Such models of sparsity-inducing penalties are often convex and can be solved by convex optimization algorithms \cite{bach2012optimization}. 

There is a wide range of significant application, including the search of connected subgraphs in networks, where  the constraint set $\mathcal{M}$  cannot be encoded by sparsity-inducing norms.  Notable application examples include disease outbreak as a connected cluster \cite{arias2011detection,qian2014connected,aksoylar2017connected}, and social events as connected dense subgraphs \cite{rozenshtein2014event}. To capture these complex graph structures such as trees and connected graphs, recent works by~\citet{baraniuk2010model} and~\citet{hegde2014fast,hegde2016fast,hegde2015nearly} have proposed to use structured sparsity model $\mathbb{M}$ to define allowed supports $\{S_1,S_2,\ldots,S_k\}$ where $\mathcal{M}$ can be represented by $\mathcal{M}(\mathbb{M})=\{{\bm x}: \text{supp}(\bm x) \subseteq S_i,\text{ for some } S_i \in \mathbb{M}\}$. These complex models are non-convex and  (batch) gradient descent algorithms involve a projection operator ${\rm P}_{\mathcal{M}(\mathbb{M})}(\cdot)$ which is usually NP-hard.  In a series of seminal work, \citet{hegde2015nearly,hegde2016fast,hegde2015approximation} used two approximated projections (head and tail) without sacrificing too much precision. However, their work only focused on the least square loss, and computing full gradient per iteration resulted in expensive per-iteration cost $\mathcal{O}(m p)$ if  $m$ is the number of examples and $p$ is the data dimension.  This largely limits its practical application to big data setting where typically the data  is contaminated by non-Gaussian noise and the data volume is huge (i.e., $m$ or/and  $p$ is very large). Therefore, it is desirable to develop efficient optimization algorithms for graph-structured sparsity constraint problems scalable to large-scale data. 

Stochastic optimization such as stochastic gradient descent (SGD) has become a standard tool for solving large-scale convex and non-convex learning problems (e.g.,~\citet{ Bach,bousquet2008tradeoffs,jin2017escape,kingma2014adam,li2018convergence,rakhlin2012making,shamir2013stochastic,yang2015rsg,ying2017unregularized,ying2006online}). SGD-type algorithms only have a cheap per-iteration computation of the gradient over one or a few training examples, which makes them amenable for large-scale data analysis. Along this research line, SGD-type algorithms for sparse models have been proposed including stochastic proximal gradient methods \cite{rosasco2014convergence,duchi2011adaptive}  and stochastic gradient hard thresholding \cite{nguyen2017linear,zhou2018efficient,murata2018sample,shen2017tight,liu2017dual}.  However, these algorithms cannot address the important case of complex (non-convex) graph-structured sparsity constraint problems. 

In this paper we leverage the recent success of SGD algorithms for non-convex problems and propose an efficient stochastic optimization algorithm for graph-structured sparsity constraint problems. In particular, we use the structure sparsity model $\mathcal{M}(\mathbb{M})$ to capture more complex graph-structured information of interest. Our main contributions can be summarized as follows:

~ $\bullet$  We propose a stochastic gradient-based algorithm for solving graph-structured sparsity constraint problems. To our best knowledge, our work is the first attempt to provide stochastic gradient descent-based algorithm for graph-structured sparsity constraint problems.

~ $\bullet$  The proposed algorithm enjoys linear convergence property under proper conditions.\footnote{Linear convergence up to a tolerance error.} It is proved applicable to a broad spectrum of loss functions. Specifically, for the least square loss, it enjoys a linear convergence rate  which is competitive to the algorithm proposed in~\citet{hegde2016fast}. For the logistic loss, we show that it also satisfies the linear convergence property with high probability. Our proofs could be easily applied to other types of projections.

~ $\bullet$  We conduct extensive experiments to validate the proposed stochastic algorithm. In particular, we focus on two applications: graph sparse linear regression and graph logistic regression on breast cancer data. Our experiment results show that the proposed  stochastic algorithm consistently outperforms the deterministic ones.

%%%%%%%%%%%%%%%%%%%%%%%%%%%%%%%%%%%%%%%%%%%%%%%%%%%%%%%%%%%%%%%%%%%%%%%%%%%%%%%%

\textbf{Outline.} We organize the rest of the paper as follows. In Section~\ref{section:preliminary}, we introduce the notations, definitions and key assumptions. Our proposed algorithm  is presented in Section~\ref{section:algorithm}. Section~\ref{section:convergence_analysis} provides the theoretical analysis of the algorithm. Numerical experiments and discussion are respectively in Section~\ref{section:experiments} and ~\ref{section:discussion}. Due to the limited space, all of our formal proofs and further experiment results are  provided in the Supplementary Material. The source code and datasets are accessible at: \url{https://github.com/baojianzhou/graph-sto-iht}.

%%%%%%%%%%%%%%%%%%%%%%%%%%%%%%%% Preliminary %%%%%%%%%%%%%%%%%%%%%%%%%%%%%%%%%%%
%%%%%%%%%%%%%%%%%%%%%%%%%%%%%%%%%%%%%%%%%%%%%%%%%%%%%%%%%%%%%%%%%%%%%%%%%%%%%%%%
\section{Preliminary}
\label{section:preliminary}
\textbf{Notation.\quad} We use boldface letters, e.g., ${\bm X}$ to denote matrices. 
The boldface lower-case letters, e.g., ${\bm x}, {\bm y}$ are column vectors. The upper-case letters, e.g., $H, T, S$ stand for subsets of $[p]:=\{1,2,\ldots, p\}$. ${\bm X}_S\in \mathbb{R}^{|S|\times p}$ is the submatrix by keeping rows only in set $S$. Given the standard basis $\{{\bm e}_i: 1\leq i\leq p\}$ of $\mathbb{R}^p$, we also use $H, T, S$ to represent subspaces. For example, the subspace induced by $H$ is the subspace defined as $\text{span}\{{\bm e}_i: i \in H \}$. We will clarify the difference only if confusion might occur. The restriction of ${\bm x}$ on $S$ is the vector ${\bm x}_S$ such that $i$-th entry $({\bm x}_S)_i= x_i$ if $i \in S$; otherwise 0. We denote the support of ${\bm x}$ as $\text{supp}({\bm x}) := \{i: x_i \ne 0\}$. We use $\|{\bm x}\|$ to denote $\ell^2$ norm of ${\bm x}$. The set complement is defined as $S^c = [p]\backslash S$. 

\begin{definition}[Subspace model~\cite{hegde2016fast}]
\label{definition_0}
Given the space $\mathbb{R}^p$, a subspace model $\mathbb{M}$ is defined as a family of linear subspaces of $\mathbb{R}^p$: $\mathbb{M} = \{S_1,S_2,\ldots, S_k,\ldots,\}$, where each $S_k$ is a subspace of $\mathbb{R}^p$. The set of corresponding vectors in these subspaces is denoted as 
\begin{equation}
\mathcal{M}(\mathbb{M}) = \{{\bm x}: {\bm x} \in V, 
\text{ for some } V \in \mathbb{M}\}.
\label{equation:structure_model}
\end{equation}
\end{definition}
\vspace{-2mm}
Equation (\ref{equation:structure_model}) is general enough to include many sparse models. For example, the $s$-sparse set $\mathcal{M}(\mathbb{M}_s)=\{{\bm x}: |\text{supp}({\bm x})| \leq s\}$~\cite{nguyen2017linear,yuan2018gradient,jain2014iterative,bahmani2013greedy}, where $\mathbb{M}_s$ contains subspaces spanned by $s$ different standard basis vectors, i.e., $\mathbb{M}_s=\{\text{span}\{{\bm e}_{i_1},{\bm e}_{i_2},\ldots,{\bm e}_{i_s}\}: \{i_1,i_2,\ldots,i_s\}\subseteq [p]\}$. Although the proposed algorithm and theorems apply to any $\mathbb{M}$ which admits efficient projections, in this paper, we mainly consider graph-structured sparsity constraint models such as the weighted graph model in~\citet{hegde2015nearly}.

\begin{definition}[Weighted Graph Model~\cite{hegde2015nearly}]
Given an underlying graph $\mathbb{G}=(\mathbb{V},\mathbb{E})$ defined on the coefficients of the unknown vector ${\bm x}$, where $\mathbb{V}=[p]$ and $\mathbb{E}\subseteq \mathbb{V}\times \mathbb{V}$, then the weighted graph model $(\mathbb{G},s,g,C)$-WGM can be defined as the following set of supports
\begin{align*}
\mathbb{M}=&\{S: |S|\leq s, \text{ there is an } 
F \subseteq \mathbb{V} \text{ with } \\
&\mathbb{V}_F = S, \gamma(F) = g, \text{ and } w(F)\leq C\},
\end{align*}
where $C$ is the budget on weight of edges $w$, $g$ is the number of connected components of $F$, and $s$ is the sparsity.
\end{definition}
$(\mathbb{G},s,g,C)$-WGM captures a broad range of graph structures such as groups, clusters, trees, and connected subgraphs. This flexibility requires an efficient projection operator onto $\mathcal{M}(\mathbb{M})$, i.e., ${\rm P}(\cdot, \mathcal{M}(\mathbb{M})): \mathbb{R}^{p} \rightarrow \mathbb{R}^p$ defined as
\begin{equation}
    {\rm P}({\bm x},\mathcal{M}(\mathbb{M})) = \argmin_{{\bm y} \in \mathcal{M}(\mathbb{M})} 
    \| {\bm x} - {\bm y} \|^2.
    \label{equ:projection_operator}
\end{equation}
The above operator is crucial for projected gradient descent-based methods, but to solve the operator is NP-hard in general. To explore~(\ref{equ:projection_operator}), one needs to solve the following equivalent minimization problem
\begin{equation*}
\min_{{\bm y} \in \mathcal{M}(\mathbb{M})} 
\| {\bm x} - {\bm y} \|^2 \Leftrightarrow 
\min_{S \in \mathbb{M} } \| {\bm x} - {\rm P}({\bm x}, S) \|^2.
\end{equation*}
Since ${\rm P}({\bm x},S)$ is an orthogonal projection operator that projects ${\bm x}$ onto subspace $S$, by projection theorem, it always has the following property
\begin{equation}
\|{\bm x} \|^2 - \| {\rm P}({\bm x}, S)\|^2 = \| {\bm x} - {\rm P}({\bm x}, S) \|^2.
\nonumber
\end{equation}
By adding minimization to both sides with respect to subspace $S$, we obtain
\begin{align}
\min_{S\in \mathbb{M}} \Big\{\|{\bm x} \|^2 - \| {\rm P}({\bm x}, S)\|^2 \Big\} 
&= \min_{S \in \mathbb{M}}\| {\bm x} - {\rm P}({\bm x}, S) \|^2 \nonumber \\
\|{\bm x} \|^2 + \max_{S\in \mathbb{M}} \| {\rm P}({\bm x},S)\|^2
&= \min_{S\in \mathbb{M}}\| {\bm x} - {\rm P}({\bm x}, S) \|^2. \label{equ:two_projection}
\end{align}
The above observations lead to a key insight that the NP-hard problem~(\ref{equ:projection_operator}) can be solved either by maximizing $\|{\rm P}({\bm x}, S)\|^2$ or by minimizing $\| {\bf x} - {\rm P}({\bm x}, S)\|^2$ over $S$. Instead of minimizing or maximizing these two squared norms over $S$ exactly, \citet{hegde2015approximation,hegde2015nearly} turned to approximated algorithms, and gave two important projections---head projection and tail projection.

%%%%%%%%%%%%%%%%%%%%%%%%%%%%%%%%%%%%%%%%%%%%%%%%%%%%%%%%%%%%%%%%%%%%%%%%%%%%%%%%
\begin{assumption}[Head Projection~\cite{hegde2016fast}] 
Let $\mathbb{M}$ and $\mathbb{M}_\mathcal{H}$ be the predefined subspace models. Given any vector ${\bm x}$, there exists a $(c_\mathcal{H}, \mathbb{M}, \mathbb{M}_\mathcal{H})$-Head-Projection which is to find a subspace $H\in \mathbb{M}_{\mathcal{H}}$ such that
\begin{equation}
\|{\rm P}({\bm x}, H) \|^2 \geq c_{\mathcal{H}} \cdot 
\max_{S \in \mathbb{M}} \| {\rm P}({\bm x}, S)\|^2,
\end{equation}
where $0 < c_\mathcal{H} \leq 1$. We denote ${\rm P}({\bm x}, H)$ as ${\rm P}({\bm x},\mathbb{M},\mathbb{M}_\mathcal{H})$.
\end{assumption}

\begin{assumption}[Tail Projection~\cite{hegde2016fast}]
Let $\mathbb{M}$  and $\mathbb{M}_\mathcal{T}$ be the predefined subspace models. Given any vector ${\bm x}$, there exists a $(c_\mathcal{T}, \mathbb{M}, \mathbb{M}_\mathcal{T})$-Tail-Projection which is to find a subspace $T\in \mathbb{M}_\mathcal{T}$ such that
\begin{equation}
\|{\rm P} ({\bm x}, T) - {\bm x} \|^2 \leq c_\mathcal{T} \cdot 
\min_{S \in \mathbb{M}} \| {\bm x} - {\rm P}({\bm x}, S)\|^2,
\end{equation}
where $c_\mathcal{T} \geq 1$. We denote ${\rm P}({\bm x}, T)$ as ${\rm P}({\bm x},\mathbb{M},\mathbb{M}_\mathcal{T})$.
\end{assumption}
Intuitively, head projection keeps large magnitudes while tail projection discards small magnitudes. Take $\mathbb{M}_s$ (a complete graph) as an example. Head projection is to find subspace $H=\text{span}\{{\bm e}_{i_1},{\bm e}_{i_2}, \ldots, {\bm e}_{i_s}\}$ where $\{i_1, i_2, \ldots, i_s\}\in [p]$ is the set of indices that keeps the largest $s$ components of ${\bm x}$, i.e, $|x_{i_1}| \geq |x_{i_2}| \geq \cdots \geq |x_{i_s}|$. As ~(\ref{equ:projection_operator}) could have multiple solutions, we could have $H\ne T$, but we must have $\|{\bm x} - {\bm x}_H\| = \|{\bm x} - {\bm x}_T\|$.
\begin{definition}[$(\alpha,\beta, \mathcal{M}(\mathbb{M}))$-RSC/RSS]
\label{definition_1}
We say a differentiable function $f(\cdot)$ satisfies the  $(\alpha,\beta, \mathcal{M}(\mathbb{M}))$-Restricted Strong Convexity (RSC)/Smoothness (RSS) property 
if there exist positive constants $\alpha$ and $\beta$ such that
\begin{align}
\frac{\alpha}{2} \| {\bm x} - {\bm y} \|^2 
\leq B_{f}({\bm x},{\bm y}) 
\leq \frac{\beta}{2} \| {\bm x} -{\bm y} \|^2,
\label{inequ:rssc}
\end{align}
for all ${\bm x},{\bm y} \in \mathcal{M}(\mathbb{M})$, where $B_{f}({\bm x}, {\bm y})$ is the Bregman divergence of $f$, i.e., $B_{f}({\bm x}, {\bm y}) = f({\bm x}) - f({\bm y}) - \langle \nabla f({\bm y}), {\bm x} -{\bm y} \rangle$. $\alpha$ and $\beta$ are strong convexity parameter and strong smoothness parameter respectively.
\end{definition}

RSC/RSS property was firstly introduced in~\citet{agarwal2012fast}\footnote{This property is related with the theory of paraconvexity~\cite{van2008paraconvex} as pointed out by~\citet{agarwal2012fast}.}. Since it captures sparsity of many functions, it has been widely used (e.g.,~\citet{jain2014iterative,yuan2018gradient,elenberg2018restricted,johnson2013accelerating,shen2017tight,zhou2018efficient,nguyen2017linear,zhou2018new}). For example, if $f$ is the least square loss, i.e., $f({\bm x}) = \| {\bm A} {\bm x} - {\bm y} \|^2$, then RSC/RSS property can be reduced to the canonical subspace Restricted Isometry Property (RIP)~\cite{candes2005decoding}. Next, we characterize the property of $F$ and each $f_i$ of~(\ref{equ:objective_function}) in Assumption~\ref{assumption:assumption_01}.

\begin{assumption}
Given the objective function $F({\bm x})$ in~(\ref{equ:objective_function}), we assume that $F({\bm x})$ satisfies $\alpha$-RSC in subspace model $\mathcal{M}(\mathbb{M}\oplus \mathbb{M}_\mathcal{H}\oplus \mathbb{M}_{\mathcal{T}})$. Each function $f_i({\bm x})$ satisfies $\beta$-RSS in $\mathcal{M}(\mathbb{M}\oplus\mathbb{M}_\mathcal{H}\oplus \mathbb{M}_{\mathcal{T}})$, where $\oplus$ of two models $\mathbb{M}_1$ and $\mathbb{M}_2$ is defined as $\mathbb{M}_1\oplus\mathbb{M}_2 := \{S_1\cup S_2 : S_1 \in \mathbb{M}_1, S_2 \in \mathbb{M}_2\}$.
\label{assumption:assumption_01}
\end{assumption}

%%%%%%%%%%%%%%%%%%%%%%%%%%%%%%%% Algorithms %%%%%%%%%%%%%%%%%%%%%%%%%%%%%%%%%%%%
%%%%%%%%%%%%%%%%%%%%%%%%%%%%%%%%%%%%%%%%%%%%%%%%%%%%%%%%%%%%%%%%%%%%%%%%%%%%%%%%
\vspace{-2mm}
\section{Algorithm: \textsc{GraphStoIHT}}
\label{section:algorithm}
\begin{algorithm}[H]
\caption{\textsc{GraphStoIHT}}
\begin{algorithmic}[1]
 \STATE \textbf{Input}: 
 $\eta_t, F(\cdot), \mathbb{M},\mathbb{M}_\mathcal{H},\mathbb{M}_\mathcal{T}$
 \STATE \textbf{Initialize}: ${\bm x}^0$ such that 
 $ \text{supp}({\bm x}) \in \mathbb{M}$ and $t=0$
 \FOR{$t=0, 1, 2, \ldots$}
 \STATE Choose $\xi_t$ from $[n]$ with probability $Pr(\xi_t)$
 \STATE ${\bm b}^t = {\rm P} (\nabla f_{\xi_t} ({\bm x}^{t}), \mathbb{M}\oplus\mathbb{M}_\mathcal{T}, {\mathbb{M}_\mathcal{H}})$
 \STATE ${\bm x}^{t+1} = {\rm P}({\bm x}^t - \eta_t {\bm b}^t,
 \mathbb{M},\mathbb{M}_\mathcal{T})$
 \ENDFOR
 \STATE \textbf{Return} $\bm x^{t+1}$
\end{algorithmic}\label{alg:graph-sto-iht}
\end{algorithm}
\vspace{-5mm}
The proposed algorithm is named \textsc{GraphStoIHT} presented in Algorithm~\ref{alg:graph-sto-iht}, a stochastic-based method for graph-structured sparsity problems. Initially, ${\bm x}^0 ={\bm 0}$ and $t=0$. At each iteration, it works as the following three steps: 

~ $\bullet$ \textbf{Step 1:} in Line 4, it randomly selects $\xi_t$ from $[n]$ with probability mass function $Pr(\xi_t)$, that is, the current loss $f_{\xi_t}(\cdot)$ has been chosen;

~ $\bullet$ \textbf{Step 2:} in Line 5, the head projection inputs the gradient of $f_{\xi_t}({\bm x})$ and outputs the projected vector ${\bm b}^t$ so that $ \text{supp}({\bm b}^t) \in \mathcal{M}(\mathbb{M}_{\mathcal{H}})$;
    
~ $\bullet$  \textbf{Step 3:} in Line 6, the next estimated ${\bm x}^{t+1}$ is then updated by using the tail projection.

The algorithm repeats the above three steps until stop condition is satisfied. The difference between \textsc{GraphStoIHT} and \textsc{StoIHT}~\cite{nguyen2017linear}, the latter of which is essentially a stochastic projected gradient descent, is two-fold: 

1) instead of directly using the gradient $\nabla f_{\xi_t}(\cdot)$ at each iteration as \textsc{StoIHT} did, \textsc{GraphStoIHT} uses a thresholded gradient. Thresholding $\nabla f_{\xi_t}(\cdot)$ at the first stage could be helpful under sparse setting, because most nonzero entries of $\nabla f_{\xi_t}(\cdot)$ are irrelevant.

2) instead of projecting on $\mathcal{M}(\mathbb{M}_s)$ as \textsc{StoIHT} did, the tail projection projects each estimation onto graph-structured subspaces $\mathcal{M}(\mathbb{M}_{\mathcal{T}})$. 

We call a batch version ( $n=1$) of this algorithm as \textsc{GraphIHT}, where we simply apply the head and tail projection to \textsc{AS-IHT} in ~\citet{hegde2016fast} by considering a general loss function.

\textbf{Time complexity Analysis.\quad} Due to the NP-hardness of the projection, there is a trade-off between the time cost and the quality of projection. The time complexity of two projections depends on the graph size $p$ and the number of edges $|\mathbb{E}|$. As proved in~\citet{hegde2015nearly}, the running time of both head projection and tail projection is bounded by $\mathcal{O}(|\mathbb{E}|\log^3(p))$. If the graph is sparse as common in real-world applications, i.e., $|\mathbb{E}| = \mathcal{O}(p)$, two projections have nearly-linear time complexity: $\mathcal{O}(p\log^3(p))$ with respect to the feature dimension $p$.  Hence, per-iteration cost $\mathcal{O}(p\log^3(p))$, is still cheaper than that of the deterministic GD algorithms \textsc{GraphIHT} in which the computation of the full gradient will be of cost $\mathcal{O}( n p)$.

%%%%%%%%%%%%%%%%%%%%%%%%%%% Convergence Analysis %%%%%%%%%%%%%%%%%%%%%%%%%%%%%%%
%%%%%%%%%%%%%%%%%%%%%%%%%%%%%%%%%%%%%%%%%%%%%%%%%%%%%%%%%%%%%%%%%%%%%%%%%%%%%%%%
\section{Convergence Analysis of \textsc{GraphStoIHT}}
\label{section:convergence_analysis}
In this section, firstly we give the convergence analysis of 
\textsc{GraphStoIHT} by characterizing the \textbf{\textit{estimation error}} between ${\bm x}^{t+1}$ and ${\bm x}^*$, i.e, $\|{\bm x}^{t+1} - {\bm x}^*\|$,  where ${\bm x}^*$ is an optimal solution of~(\ref{equ:objective_function}). Then we analyze two commonly used objective functions and prove our algorithm achieves a linear convergence up to a constant error. Our analysis is applicable not only to graph-structured sparsity model but also to any other head and tail projection. 

We denote  $\xi_{[t]} = (\xi_0,\xi_1,\ldots,\xi_t)$ as the history of stochastic process $\xi_0,\xi_1,\ldots$ up to time $t$, and all random variables $\xi_t$ are independent of each other. Define the \textit{inverse condition number} $\mu:=\alpha/\beta$. To simplify the analysis, we define the probability mass function $Pr(\xi_t=i)=1/n, 1\leq i \leq n$. Before presenting our main Theorem~\ref{theorem:theorem_01}, let's take a look at the following key lemma.
\vspace{-2mm}
\begin{lemma}
If each $f_{\xi_t}(\cdot)$ and $F({\bm x})$ satisfy Assumption~\ref{assumption:assumption_01} and, given head projection model $(c_\mathcal{H}, \mathbb{M}\oplus\mathbb{M}_\mathcal{T}, \mathbb{M}_\mathcal{H})$ and tail projection model $(c_\mathcal{T}, \mathbb{M}, \mathbb{M}_\mathcal{T})$, then we have the following inequality
\vspace{-4mm}
\begin{equation}
\mathbb{E}_{\xi_t}\|({\bm x}^t - {\bm x}^*)_{H^c}\|  
\leq \sqrt{1 - \alpha_0^2} \mathbb{E}_{\xi_t}\|{\bm x}^t - {\bm x}^*\| 
+ \sigma_1, \label{inequ:7}
\vspace{-4mm}
\end{equation}
where 
\begin{align*}
\sigma_1 &=\Big( \frac{\beta_0}{\alpha_0} + \frac{\alpha_0 \beta_0}{\sqrt{1-\alpha_0^2}} 
\Big)\mathbb{E}_{\xi_t}\|\nabla_I f_{\xi_t}({\bm x}^*)\|,\\
H &= \text{supp}({\rm P}(\nabla f_{\xi_t}({\bm x}^t), 
\mathbb{M} \oplus \mathbb{M}_\mathcal{T}, 
\mathbb{M}_{\mathcal{H}} )),\\
\alpha_0 &= c_\mathcal{H} \alpha \tau - 
\sqrt{\alpha\beta \tau^2 - 2\alpha\tau + 1}, \quad \beta_0 = (1+c_\mathcal{H})\tau,\\
I &= \argmax_{S \in \mathbb{M} \oplus 
\mathbb{M}_\mathcal{T} \oplus \mathbb{M}_\mathcal{H} } 
\mathbb{E}_{\xi_t}\|\nabla_S f_{\xi_t}({\bm x}^*)\|, \text{ and } \tau \in (0, 2/\beta).
\end{align*}
\label{lemma:lemma1}
\end{lemma}
\vspace{-5mm}
Inequality~(\ref{inequ:7}) means that the head projection always excludes a small fraction of the residual vector ${\bm x}^t - {\bm x}^*$  at each iteration. That is to say, most of the large magnitudes in  ${\bm x}^t - {\bm x}^*$ are captured by the head projection. More specifically, the head projection thresholds small magnitude entries of $\nabla f_{\xi_t}({\bm x}^t)$ to zeros. These small magnitude entries could lead the algorithm to a wrong direction. When $F({\bm x})$ is the least square loss, Lemma~\ref{lemma:lemma1} here is similar to Lemma 13 in~\citet{hegde2016fast}. However, there are two important differences: 1) our Lemma~\ref{lemma:lemma1} can be applied to any functions that satisfy Assumption~\ref{assumption:assumption_01} above,  while the RIP condition used in~\citet{hegde2016fast} can  be only applied to the least square loss; 2) since each $f_{\xi_t}({\bm x})$ is not strongly convex, the proof in~\citet{hegde2016fast} cannot be directly used. Instead, we use \textit{co-coercivity} in~\citet{nguyen2017linear} and then obtain the main theorem below.
\begin{theorem}
Let ${\bm x}^0$ be the start point of Algorithm~\ref{alg:graph-sto-iht}. If we choose a constant learning rate with $\eta_t = \eta$ and use Lemma~\ref{lemma:lemma1}, then the solution ${\bm x}^{t+1}$ of Algorithm~\ref{alg:graph-sto-iht} satisfies
\begin{equation}
\mathbb{E}_{\xi_{[t]}} \|{\bm x}^{t+1} - {\bm x}^* \| 
\leq \kappa^{t+1} \| {\bm x}^0 - {\bm x}^* \| + \frac{\sigma_2}{1-\kappa},
\label{theorem:inequ_01}
\end{equation}
where 
\vspace{-4mm}
\begin{align*}
\kappa &= (1 + c_\mathcal{T})
\Big(\sqrt{ \alpha\beta\eta^2 - 2 \alpha\eta + 1} + 
\sqrt{1 - \alpha_0^2}\Big),\\
\sigma_2 &= \sigma_1 + 
\eta \mathbb{E}_{\xi_{t}}\| \nabla_{I} f_{\xi_t} ({\bm x}^*) \|, \text{ and } \eta,\tau \in (0,2/\beta).
\end{align*}
\label{theorem:theorem_01}
\end{theorem}
\vspace{-5mm}
Theorem~\ref{theorem:theorem_01} shows that Algorithm~\ref{alg:graph-sto-iht} still possibly enjoys linear convergence even if the full gradient is not available.  In order to make more detailed analysis of Theorem~\ref{theorem:theorem_01}, we call $\kappa$ in~(\ref{theorem:inequ_01}) the \textbf{\textit{contraction factor}}. A natural question  to ask is: under what condition can we get an effective contraction factor, $\kappa$, i.e., $\kappa < 1.$  When $\kappa < 1$, the estimation error is dominated by the term $\frac{\sigma_2}{1-\kappa}$. Suppose the random distribution of $\xi_t$ is uniform and the approximation factor of the head projection $c_\mathcal{H}$ can be arbitrarily boosted close to 1. Taking $\eta = 1 / \beta$, in order to get an effective contraction factor, $\kappa < 1$, the inverse of the condition number $\mu $ and tail projection factor $c_\mathcal{T}$ need to satisfy
\begin{equation}
\kappa =(1+c_\mathcal{T})(1+2\sqrt{\mu})\sqrt{ 1 - \mu } < 1.
\label{equation:effective_kappa}
\end{equation}
To be more specific, if $\kappa$ is a function of $\eta$, then it takes minimum at $\eta = 1/\beta$. By letting $\eta=1/\beta$, then $\kappa$ is simplified as $(1 + c_\mathcal{T}) (\sqrt{ 1 - \mu}  + \sqrt{1-\alpha_0^2})$. Furthermore, $\alpha_0$ is a concave function with respect to $\tau$ (recall that $\tau$ is a free parameter in $(0,2/\beta)$), and then we have its maximum at $\tau=\frac{1}{\beta}(1 + c_\mathcal{H} \sqrt{\frac{\beta - \alpha}{\beta - c_\mathcal{H}^2\alpha}})$ after calculation. By taking $c_\mathcal{H}\rightarrow 1^{-}$, we present $\kappa$ as $(1+c_\mathcal{T}) (1+2\sqrt{\mu})\sqrt{1- \mu}$

One of the assumptions of~(\ref{equation:effective_kappa}) is that $c_\mathcal{H}$ can be arbitrarily close to 1. One can boost the head projection by using the method proposed in~\citet{hegde2016fast}. In our experiment, we find that it is not necessary to boost the head projection because executing the head projection once is sufficient enough to obtain good performance. In the remainder of the section, we consider two popular loss  functions to discuss the conditions of getting effective contraction factors.

\textbf{Graph sparse linear regression.\quad} Given a design matrix ${\bm A}\in \mathbb{R}^{m\times p}$ and corresponding observed noisy vector ${\bm y} \in\mathbb{R}^m$ that are linked via the linear relationship
\begin{equation}
{\bm y} = {\bm A} {\bm x}^* + {\bm \epsilon},
\label{objective:least_square_model}
\end{equation}
where $\bm \epsilon\sim \mathcal{N}(\bm 0,\sigma^2\bm I)$. The graph sparse linear regression is to estimate the underlying sparse vector ${\bm x}^*$. The underlying graph $\mathbb{G}(\mathbb{V},\mathbb{E})$ is defined on ${\bm x}^*$. For example, $\text{supp}( {\bm x}^*)$ induces a connected subgraph with $s$ nodes, showcased in Figure~\ref{fig:simu_figs_re_00} in the experiment section. To estimate ${\bm x}^*$, naturally we consider the least square loss and formulate it as
\begin{align}
\argmin_{\text{supp}({\bm x}) \in \mathcal{M}(\mathbb{M})} F({\bm x}) := \frac{1}{n} \sum_{i=1}^{n} \frac{n}{2m} \| {\bm A}_{B_i} {\bm x} - {\bm y}_{B_i}\|^2,
\label{objective:least_square}
\end{align}
where $m$ observations have been partitioned into $n$ blocks, $B_1, B_2,\ldots, B_n$. Each block $B_i$ is indexed by $i$ with block size $b=m/n$. In Corollary~\ref{corollary:corollary_01}, we show that it is possible to get a linear convergence rate under RSC/RSS assumption. 

\begin{corollary}
If $F({\bm x})$ in ~(\ref{objective:least_square}) satisfies the $\alpha$-RSC property and each $f_i ({\bm x})=\frac{n}{2m}\|{\bm A}_{B_i} {\bm x} - {\bm y}_{B_i}\|^2$ satisfies the $\beta$-RSS property, then we have the following condition
\begin{equation}
\frac{\alpha}{2} \|{\bm x}\|^2 \leq \frac{1}{2m}\| 
{\bm A}{\bm x} \|^2 \leq \frac{\beta}{2} \|{\bm x}\|^2. \nonumber
\end{equation}
Let the strong convexity parameter $\alpha = 1-\delta$ and strong smoothness parameter $\beta=1+\delta$, where $0<\delta <1$. The condition of effective contraction factor is as the following
\begin{equation*}
(1+c_\mathcal{T})\Big(\sqrt{\frac{2}{1+\delta}} + \frac{2\sqrt{2(1-\delta)}}{1+\delta}\Big)\sqrt{\delta} < 1.
\end{equation*}
\label{corollary:corollary_01}
\vspace*{-0.2in}
\end{corollary}
Table~\ref{table:table_01} gives a contrast of contraction factor $\kappa$ in our method and in~\citet{hegde2016fast}. Both conditions are obtained by letting $c_\mathcal{H}\rightarrow 1^{-}$. 
\begin{table}[ht!]
\def\arraystretch{1.5}
\caption{$\kappa$ of two algorithms}
\centering
\begin{tabular}{cc}
\hline\hline
\textbf{Algorithm} & $\kappa$ \\
\hline
\textsc{GraphIHT} & $(1+c_\mathcal{T})\Big(\sqrt{\delta}+ 2\sqrt{1-\delta}\Big)\sqrt{\delta}$ \\
\hline
\textsc{GraphStoIHT} & $(1+c_\mathcal{T})\Big(\sqrt{\frac{2}{1+\delta}} + \frac{2\sqrt{2(1-\delta)}}{1+\delta}\Big)\sqrt{\delta}$ \\
\hline
\label{table:table_01}
\end{tabular}
\vspace*{-6mm}
\end{table}
Table~\ref{table:table_01} shows that $\kappa$ of the two algorithms could be effective if $\delta\rightarrow 0$. By using a sufficiently large number of observations $m$, we always have $\delta \rightarrow 0$. 
According to the preceding analysis, \textsc{GraphIHT} uses a constant learning rate $\eta=1$ while \textsc{GraphStoIHT} uses the learning rate $\eta=1/(1+\delta)$. However, the difference between the two learning rates disappears when $\delta \rightarrow 0$. To be more specific, $\kappa$ of \textsc{GraphIHT} is controlled by $\mathcal{O}(\sqrt{\delta}\cdot 2(1+c_\mathcal{T}))$ while for \textsc{GraphStoIHT}, $\kappa$ is controlled by $\mathcal{O}(\sqrt{\delta}\cdot 3\sqrt{2}(1+c_\mathcal{T}))$. Consequently,  the contraction factor $\kappa$ of \textsc{GraphStoIHT} is competitive with that of \textsc{GraphIHT} while \textsc{GraphStoIHT} has the advantage of cheaper per-iteration cost $\mathcal{O}( mp/n)$, if the size of blocks $n$ is large, compared to $\mathcal{O}(m p)$ of \textsc{GraphIHT}. To obtain $\kappa < 1$, $\delta \leq 0.0527$ for \textsc{GraphIHT} while $\delta \leq 0.0142$ for \textsc{GraphStoIHT}. The gap between the two $\kappa$ is mainly due to the randomness introduced in our algorithm. 

The error term $\sigma_2$ in Equation~(\ref{theorem:inequ_01}) mainly depends on $b$ and $\sigma$. The gradient of $f_{\xi_t}$ at $\bm x^*$ is $\bm A_{B_{\xi_t}}^\top \bm \epsilon_{B_{\xi_t}}/b$. As shown in~\citet{nguyen2017linear}, it can be bounded as $C \sqrt{\sigma^2|I|\log p/b}$, where $C$ is a constant independent of $b$ and $\sigma$.

Though \textsc{GraphStoIHT} needs more observations than \textsc{GraphIHT} theoretically, our experiment results demonstrate that \textsc{GraphStoIHT} uses less observations than \textsc{GraphIHT} to estimate ${\bm x}^*$ accurately under the same condition. Such kind of interesting experiment phenomenon is reported in~\citet{nguyen2017linear} but for sparsity constraint. 

\textbf{Graph logistic regression.\quad} Given a dataset $\{{\bm a_i}, y_i\}_{i=1}^m$, the graph logistic regression is formulated as the following problem
\vspace{-4mm}
\begin{align}
\argmin_{\text{supp}({\bm x}^*) \in \mathcal{M}(\mathbb{M})} F({\bm x}) := \frac{1}{n} \sum_{i=1}^{n} f_i({\bm x}),
\label{objective:logistic_regression}
\end{align}
where each $f_i({\bm x})$ is defined as
\begin{small}
\begin{equation}
f_i({\bm x}) =\frac{n}{m} \sum_{j=1}^{m/n}  \Big[ \log (1 + \exp{(- y_{i_j} {\bm a}_{i_j}^\mathsf{T}{\bm x})}) \Big] + \frac{\lambda}{2} \| {\bm x} \|^2.
\label{objective:logistic_regression_i}
\end{equation}
\end{small}
\noindent Here $\lambda$ is the regularization parameter. Again we divide $[m]$ into $n$ blocks, i.e., $B_1, B_2,\ldots,B_n$ and assume each block size $b$ is the same, i.e., $b =m/n$. Problem (\ref{objective:logistic_regression}) has many important applications. For example, in its application to breast cancer classification based on gene mirco-array dataset and gene-gene network, each ${\bm a}_i$ is the $i$-th sample representing the gene expression records of patient $i$, and $y_i$ is associated label with value -1 or 1.  $y_i=1$ means metastatic growth on patient $i$ while $y_i=-1$ non-metastatic growth on patient $i$. Assume that we have more prior information such as gene-gene network, the goal here is to minimize the objective function $F({\bm x})$ meanwhile to find a connected subgraph induced by $\text{supp}({\bm x})$ which is highly relevant with breast cancer-related genes. This kind of application is theoretically based on the following corollary.
\begin{corollary}
Suppose we have the logistic loss, $F({\bm x})$ in~(\ref{objective:logistic_regression}) and each sample ${\bm a}_i$ is normalized, i.e., $\| {\bm a}_i \| = 1$. Then $F({\bm x})$ satisfies $\lambda$-RSC and each $f_{i} ({\bm x})$ in~(\ref{objective:logistic_regression_i}) satisfies $(\alpha + (1+\nu) \theta_{max})$-RSS. The condition of getting effective contraction factor of \textsc{GraphStoIHT} is
as the following
\begin{align}
\frac{\lambda}{\lambda + n(1+\nu)\theta_{max} / 4m} \geq \frac{243}{250},
\label{condition:logistic_regresison}
\end{align}
with probability $1-p \exp{(-\theta_{max}\nu/4)}$, where $\theta_{max} = \lambda_{max} (\sum_{j=1}^{m/n} \mathbb{E}[{\bm a}_{i_j}{\bm a}_{i_j}^\mathsf{T}])$ and $\nu \geq 1$.
\end{corollary}

Given a sufficiently large number of observations $m$ and a suitable $n$, it is still possible to find a bound such that the inequality~(\ref{condition:logistic_regresison}) is valid. In~\citet{bahmani2013greedy}, they show when $\mu > 1/(1+\frac{(1+\nu)\theta_{\max}}{4\lambda})$, it is sufficient for \textsc{GraSP}~\cite{bahmani2013greedy} to get a linear convergence rate. Our corollary shares the similar spirit.

%%%%%%%%%%%%%%%%%%%%%%%%%%%%%%% Experiments %%%%%%%%%%%%%%%%%%%%%%%%%%%%%%%%%%%%
%%%%%%%%%%%%%%%%%%%%%%%%%%%%%%%%%%%%%%%%%%%%%%%%%%%%%%%%%%%%%%%%%%%%%%%%%%%%%%%%
\section{Experiments}
\label{section:experiments}
We conduct experiments on both synthetic and real datasets to evaluate the performance of \textsc{GraphStoIHT}\footnote{Implementation details of the head and tail projection are provided in Supplementary material.}. We consider two applications, graph sparse linear regression and breast cancer metastatic classification. All codes are written in Python and C language. All experiments are tested on 56 CPUs of Intel Xeon(R) E5-2680 with 251GB of RAM.

%%%%%%%%%%%%%%%%%%%%%%%%%%%%%%%%%%%%%%%%%%%%%%%%%%%%%%%%%%%%%%%%%%%%%%%%%%%%%%%%
\subsection{Graph sparse linear regression}
\label{section:section_5.1}
\textbf{Experimental setup.\quad} The graph sparse linear regression is to recover a graph-structured vector ${\bm x}^*\in \mathbb{R}^p$ by using a Gaussian matrix ${\bm A}\in \mathbb{R}^{m\times p}$, and the observation vector ${\bm y}$ where ${\bm y}$ is measured by ${\bm y} = {\bm A} {\bm x}^* + {\bm \epsilon}$ as defined in~(\ref{objective:least_square_model}). The entries of the design matrix ${\bm A}$ are sampled from $\mathcal{N}(0,1/\sqrt{m})$ independently and nonzero entries of ${\bm x}^*$ from $\mathcal{N}(0,1)$ independently for the simulation study. ${\bm \epsilon}$ is potentially a Gaussian noise vector and ${\bm \epsilon}={\bm 0}$ the noiseless case.  We mainly follow the experimental settings in~\citet{nguyen2017linear}. All results are averaged on 50 trials of trimmed results by excluding the best $5\%$ and the worst $5\%$. All methods terminate when $\|{\bm A} {\bm x}^{t+1} - {\bm y}\|\leq 10^{-7}$ (corresponding to convergence ) or $t/n\geq 500$ (corresponding to the maximum number of epochs allowed). For \textsc{GraphStoIHT}, one epoch contains $n$ iterations. We recall that ${\bm A}$ has been partitioned into $n$ blocks with block size $b$. We say ${\bm x}^*$ is successfully recovered if the \textit{estimation error} $\|{\bm x}^{t+1} - {\bm x}^*\| \leq 10^{-6}$.
\begin{figure}[ht]
\hfill
\subfigure[$s=20$]{\includegraphics[width=3cm]{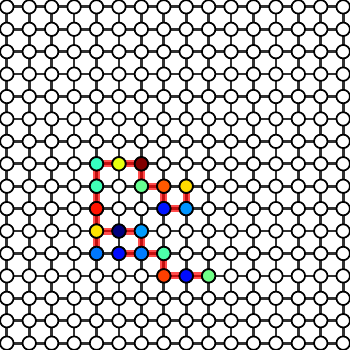}}
\hfill
\subfigure[$s=28$]{\includegraphics[width=3cm]{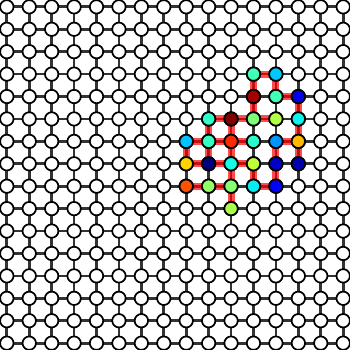}}
\hfill
\caption{
Two simulated graphs. Subgraphs (colored nodes and edges) are generated by \textit{random walk}. Each node $v_i$ is associated with $x_i$. Colored-face nodes have values sampled from $\mathcal{N}(0, 1)$ while white-face nodes have values 0.0. }
\label{fig:simu_figs_re_00}
\end{figure}

To simulate a graph structure on ${\bm x}^*\in \mathbb{R}^p$ that mimics realistic subgraph patterns (e.g., malicious activities) in a network~\cite{yu2016survey}, we fix $p=256$ and generate its $s$ nonzero entries by using \textit{random walk}~\cite{lovasz1993random} on a $16 \times 16$ grid graph so that $\|{\bm x}^*\|_0$ forms a connected subgraph.  Each edge has unit weight $1.0$. Specifically, the procedure of \textit{random walk} follows three main steps: 1) select an initial node (we choose the center of the grid); 2) move to its neighbor with probability $1/d(v_t)$, where $d(v_t)$ is the degree of node $v_t$; 3) repeat 2) until $s$ different nodes have been visited. Figure~\ref{fig:simu_figs_re_00} presents two \textit{random walks} of $s=20$ (on the left) and $s=28$ (on the right).

\begin{figure}[ht!]
\centering
{\includegraphics[width=8cm,height=3cm]{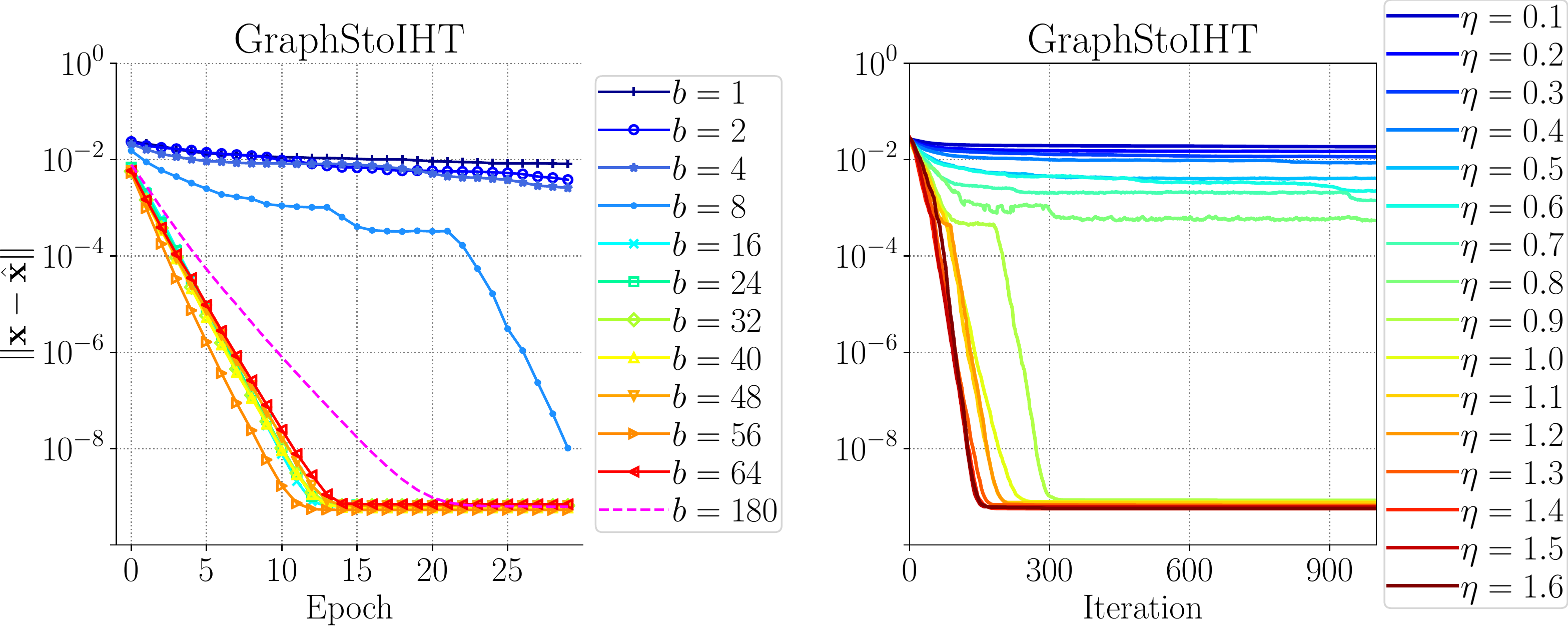}}
\caption{
Choice of $b$ and $\eta$. The left part illustrates the \textit{estimation error} as a function of epochs for different choices of $b$. When $b=180$, it degenerates to \textsc{GraphIHT} (the dashed line). The right part shows the \textit{estimation error} as a function of iterations for different choices of $\eta$.}
\label{fig:simu_figs_re_02}
\vspace*{-0.1in}
\end{figure} 

\textbf{Choice of $b$ and $\eta$.\quad} We first consider how block size $b$ and learning rate $\eta$ affect the performance of \textsc{GraphStoIHT}. We fix the sparsity $s=8$ and $n=180$ and try different  $b$ from set $\{ 1, 2, 4, 8, 16, 24, 32,40, 48, 56, 64, 180 \}$. When $b=1$, only 1 observation has  been used at each iteration; when $b=180$, all measurements used, corresponding to \textsc{GraphIHT}. Results are presented on the left of  Figure~\ref{fig:simu_figs_re_02}. In order to successfully recover ${\bm x}^*$ within 30 epochs, the block size should be at least $s$. A potential explanation is that, since $b < s$, the Hessian matrix of $f_{\xi_t}$, i.e., ${\bm A}_{B_{\xi_t}}^\top{\bm A}_{B_{\xi_t}}$ is not positive definite. Thus it is hard to converge to ${\bm x}^*$ in a short time. Another interesting finding is that \textsc{GraphStoIHT} converges faster than \textsc{GraphIHT} when block size is suitable, say between 32 and 64.

To further explore learning rate $\eta$, we use the similar setting above except $b=8, m=80$ (when $m=80$, \textsc{GraphIHT} successfully recovers ${\bm x}^*$ with high probability as shown in Figure~\ref{fig:simu_figs_re_01} (a). We consider 16 different $\eta$ from set $\{0.1, 0.2,\ldots,1.5, 1.6\}$. The right part of Figure~\ref{fig:simu_figs_re_02} shows that \textsc{GraphStoIHT} converges even when the learning rate is relatively large say $\eta \geq 1.5$. It means we can choose relatively larger learning rate so that the algorithm can achieve optimal solution faster. By Theorem~\ref{theorem:theorem_01}, the optimal learning rate $\eta$ is chosen by $\eta = 1/\beta$, so a potential explanation is that the strong smoothness parameter $\beta$ is relatively small due to the head projection inequality $\| \nabla_{H} f_{\xi_t}({\bm x}^{t+1}) -\nabla_{H} f_{\xi_t}({\bm x}^*) \|\leq \beta \|{\bm x}^{t+1} - {\bm x}^*\|$. 

\begin{figure}[ht!]
\centering
{\includegraphics[width=8.0cm,height=4.0cm]{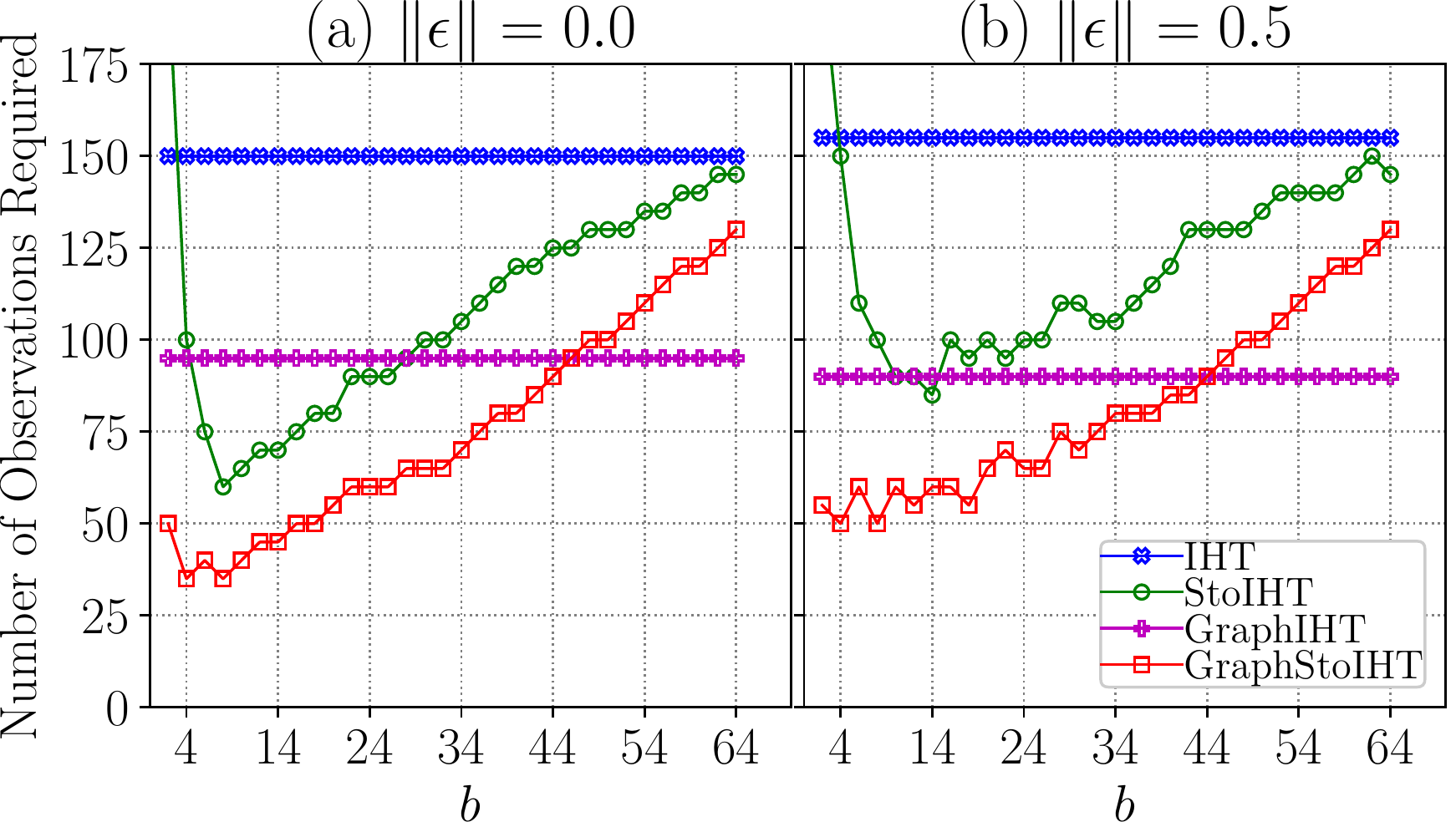}}
\vspace*{-0.1in}
\caption{Robustness to noise ${\bm \epsilon}$. The number of observations required is a function of different block sizes.}
\label{fig:simu_figs_re_04}
\vspace*{-0.1in}
\end{figure}
\textbf{Robustness to noise ${\bm \epsilon}$.\quad} To explore the performance under noise setting, we consider two noise conditions: $\|{\bm \epsilon}\| = 0.0$ (without noise) and $\|{\bm \epsilon}\| = 0.5$, where ${\bm \epsilon} \sim \mathcal{N}({\bm 0},{\bm I})$.  When $\| {\bm \epsilon} \| = 0.5$, the sparse vector is successfully recovered if $\|{\bm x}^{t+1} - {\bm x}^*\| \leq 0.5$. We try the block sizes $b$ from the set $\{2, 4, 6, 8, 10, \ldots, 62, 64\}$, and then measure the number of observations $m$ required by the algorithms (so that ${\bm x}^*$ can be successfully recovered with probability 1.0). The minimum number of observations is required such that the recovery error  $\|{\bm x}^{t+1} - {\bm x}^* \| \leq 10^{-6}$ for $\| {\bm \epsilon}\| = 0$ and $\|{\bm x}^{t+1} - {\bm x}^* \| \leq 0.5$ for $\| {\bm \epsilon}\| = 0.5$ in all trials. The results are reported in Figure~\ref{fig:simu_figs_re_04}. We can see both \textsc{GraphStoIHT} and \textsc{StoIHT} are robust to noise. In particular, when block size is between 4 and 8, the number of observations required is the least for our method. Compared with the noiseless (the left), the performance of noise case (the right) degrades gracefully.

\begin{figure}[ht!]
\centering
{\includegraphics[width=8cm,height=6cm]{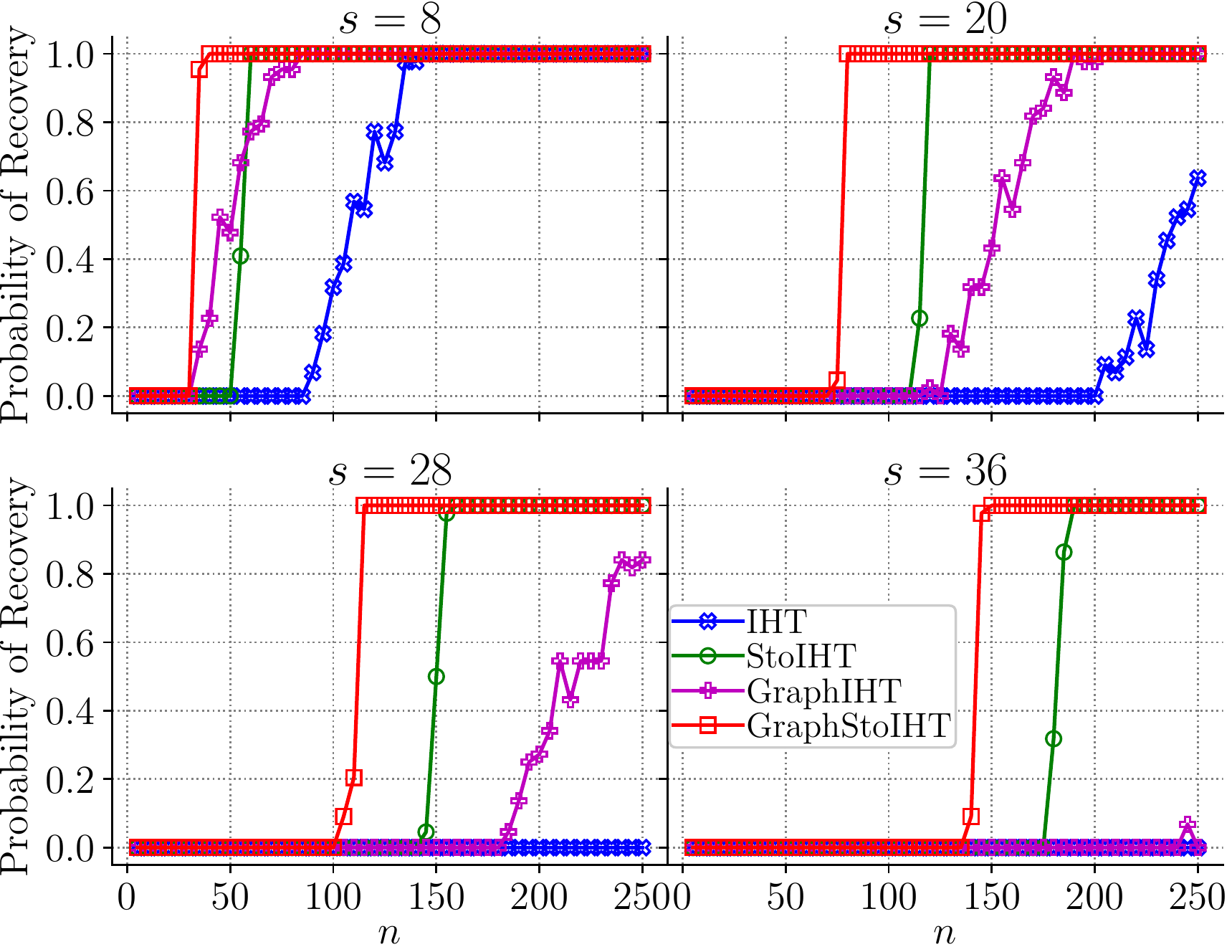}}
\vspace*{-0.1in}
\caption{Probability of recovery on synthetic dataset. The probability of recovery is a function of number of observations $m$.}
\label{fig:simu_figs_re_01}
\vspace*{-0.1in}
\end{figure}

\textbf{Results from synthetic dataset.\quad} We explore the performance of \textsc{GraphStoIHT} on the \textit{probability of recovery}, which is defined as the total number of successful trials divided by the total number of trimmed trials. Recall that ${\bm x}^*$ is successfully recovered if the \textit{estimation error} $\|{\bm x}^{t+1} - {\bm x}^* \| \leq 10^{-6}$. We compare \textsc{GraphStoIHT} with three baseline methods, i.e., Iterative Hard Thresholding (\textsc{IHT})~\cite{blumensath2009iterative}, Stochastic Iterative Hard Thresholding (\textsc{StoIHT})~\cite{nguyen2017linear}, and \textsc{GraphIHT}~\cite{hegde2016fast}. We consider four different sparsity levels, i.e., $s\in \{8,20,28,36\}$. For each trial, ${\bm x}^*$ is generated by \textit{random walk}. To be consistent with the setting in~\citet{nguyen2017linear}, the block size is set by $b=\min(s,m)$ with the number of observations $m$ chosen from $\{5, 10, \ldots, 245, 250\}$. All of the four methods, including ours, use a constant learning rate $\eta =1$. Our method uses less observations to successfully recover ${\bm x}^*$ due to the randomness (compared with \textsc{IHT} and \textsc{GraphIHT}) and the graph-structured projection (compared with \text{IHT} and \textsc{StoIHT}) as shown in Figure~\ref{fig:simu_figs_re_01}. It indicates that \textsc{GraphStoIHT} outperforms the other three baselines in terms of probability of recovery. It also shows that SGD-based methods are more stable than batch methods with respect to small perturbation of data as recently shown by the works of~\citet{hardt2016train} and~\citet{charles2018stability}. 

\begin{figure}[ht!]
\centering
{\includegraphics[width=8.0cm,height=5cm]{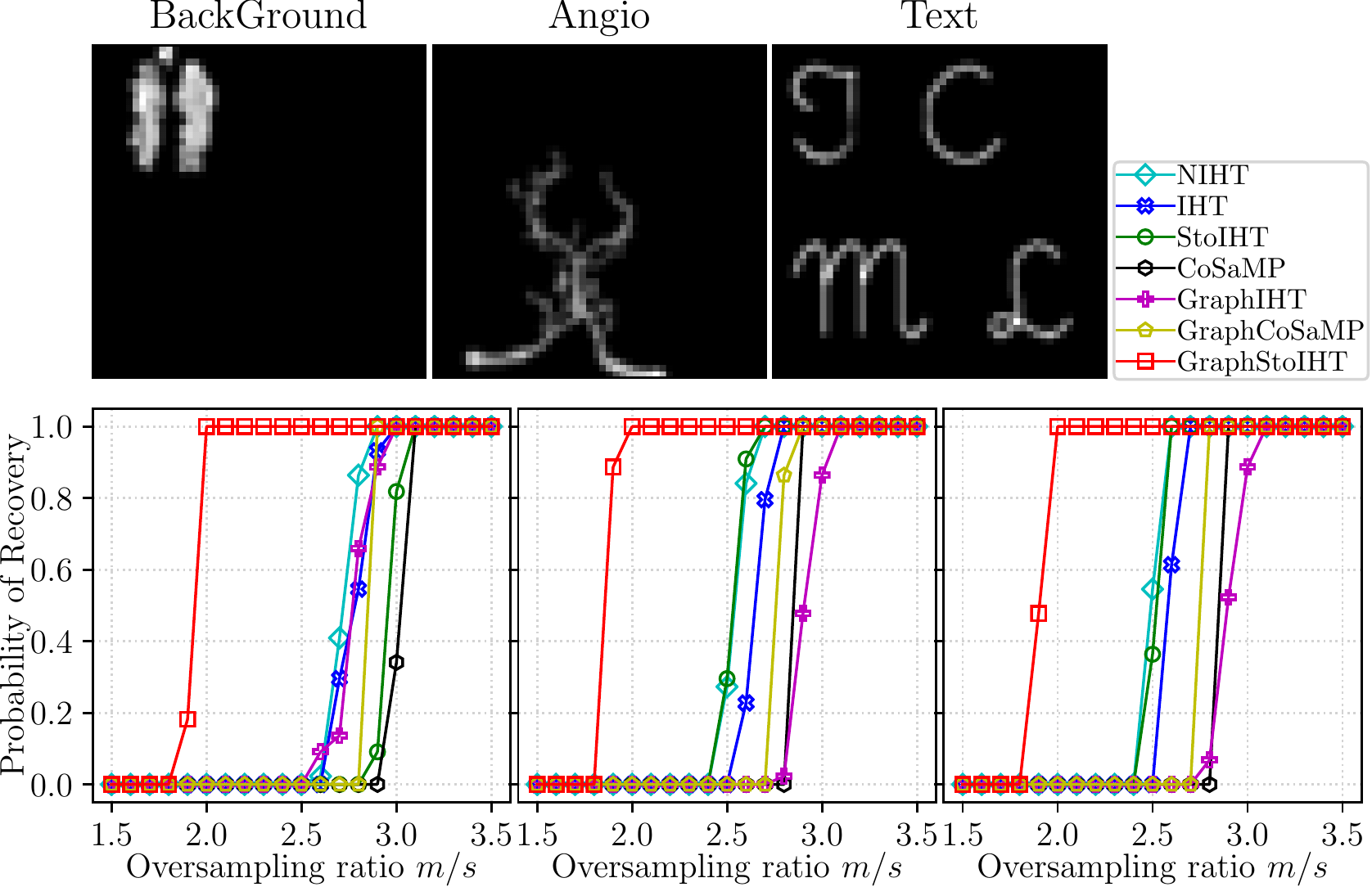}}
\caption{Probability of recovery on three $50\times 50$ resized real images: (a) Background, (b) Angio, and (c) Text~\cite{hegde2015nearly}. The probability of recovery is a function of the oversampling ratio $m/s$.}
\label{fig:simu_figs_re_07}
\vspace*{-0.1in}
\end{figure}

\begin{table*}[ht!]
\caption{AUC score $\pm$ standard deviation on the breast cancer dataset.}
\centering
\scriptsize
\begin{tabular}{ccccccccccc}
\hline
Folding ID & $\ell^1$-\textsc{Pathway} & $\ell^1/\ell^2$-\textsc{Pathway} & $\ell^1$-\textsc{Edge} & $\ell^1/\ell^2$-\textsc{Edge} 
& \textsc{IHT} & \textsc{StoIHT} & \textsc{GraphIHT} & \textsc{GraphStoIHT} \\
\hline
Folding 00 & 0.705$\pm$0.09 & 0.715$\pm$0.07 & 0.726$\pm$0.07 & 0.724$\pm$0.07 & 0.726$\pm$0.07 & \textbf{0.731}$\pm$0.06 & 0.716$\pm$0.02 & 0.718$\pm$0.03 \\
Folding 01 & 0.665$\pm$0.11 & \textbf{0.718}$\pm$0.03 & 0.688$\pm$0.04 & 0.703$\pm$0.07 & 0.680$\pm$0.07 & 0.683$\pm$0.07 & 0.710$\pm$0.06 & 0.704$\pm$0.05 \\
Folding 02 & 0.640$\pm$0.07 & \textbf{0.729}$\pm$0.06 & 0.714$\pm$0.04 & 0.717$\pm$0.04 & 0.716$\pm$0.06 & 0.723$\pm$0.06 & 0.724$\pm$0.07 & 0.720$\pm$0.08 \\
Folding 03 & 0.691$\pm$0.05 & 0.705$\pm$0.05 & 0.720$\pm$0.05 & \textbf{0.722}$\pm$0.04 & 0.699$\pm$0.05 & 0.703$\pm$0.05 & 0.687$\pm$0.04 & 0.687$\pm$0.04 \\
Folding 04 & 0.680$\pm$0.05 & 0.690$\pm$0.05 & 0.705$\pm$0.06 & \textbf{0.721}$\pm$0.05 & 0.694$\pm$0.05 & 0.695$\pm$0.06 & 0.687$\pm$0.05 & 0.688$\pm$0.03 \\
Folding 05 & 0.648$\pm$0.07 & 0.694$\pm$0.04 & 0.658$\pm$0.06 & 0.683$\pm$0.07 & 0.671$\pm$0.07 & 0.671$\pm$0.07 & 0.675$\pm$0.07 & \textbf{0.712}$\pm$0.05 \\
Folding 06 & 0.682$\pm$0.04 & 0.733$\pm$0.04 & 0.701$\pm$0.05 & 0.701$\pm$0.06 & 0.712$\pm$0.07 & 0.733$\pm$0.06 & \textbf{0.742}$\pm$0.06 & 0.741$\pm$0.06 \\
Folding 07 & 0.674$\pm$0.04 & 0.682$\pm$0.07 & 0.695$\pm$0.04 & 0.704$\pm$0.03 & 0.711$\pm$0.08 & 0.704$\pm$0.07 & \textbf{0.725}$\pm$0.07 & 0.714$\pm$0.08 \\
Folding 08 & 0.686$\pm$0.06 & 0.705$\pm$0.06 & 0.696$\pm$0.07 & 0.691$\pm$0.07 & 0.724$\pm$0.07 & 0.720$\pm$0.07 & 0.703$\pm$0.03 & \textbf{0.729}$\pm$0.03 \\
Folding 09 & 0.671$\pm$0.07 & 0.690$\pm$0.07 & 0.660$\pm$0.05 & 0.687$\pm$0.03 & \textbf{0.712}$\pm$0.06 & 0.712$\pm$0.06 & 0.703$\pm$0.06 & 0.704$\pm$0.06 \\
Folding 10 & 0.693$\pm$0.09 & \textbf{0.735}$\pm$0.09 & 0.706$\pm$0.10 & 0.718$\pm$0.06 & 0.701$\pm$0.08 & 0.717$\pm$0.09 & 0.710$\pm$0.08 & 0.707$\pm$0.07 \\
Folding 11 & 0.669$\pm$0.04 & 0.697$\pm$0.07 & 0.711$\pm$0.05 & 0.704$\pm$0.05 & 0.707$\pm$0.05 & 0.706$\pm$0.04 & \textbf{0.733}$\pm$0.06 & 0.733$\pm$0.06 \\
Folding 12 & 0.670$\pm$0.05 & \textbf{0.716}$\pm$0.06 & 0.703$\pm$0.05 & 0.701$\pm$0.03 & 0.714$\pm$0.06 & 0.715$\pm$0.07 & 0.711$\pm$0.07 & 0.711$\pm$0.08 \\
Folding 13 & 0.678$\pm$0.07 & 0.688$\pm$0.04 & 0.703$\pm$0.04 & 0.697$\pm$0.04 & 0.699$\pm$0.06 & 0.701$\pm$0.06 & 0.703$\pm$0.04 & \textbf{0.715}$\pm$0.05 \\
Folding 14 & 0.653$\pm$0.02 & 0.700$\pm$0.02 & 0.703$\pm$0.03 & 0.692$\pm$0.03 & 0.695$\pm$0.04 & 0.699$\pm$0.04 & \textbf{0.721}$\pm$0.06 & 0.721$\pm$0.06 \\
Folding 15 & 0.663$\pm$0.07 & 0.682$\pm$0.07 & 0.697$\pm$0.08 & 0.687$\pm$0.07 & 0.724$\pm$0.07 & 0.712$\pm$0.06 & \textbf{0.725}$\pm$0.07 & 0.716$\pm$0.07 \\
Folding 16 & 0.687$\pm$0.07 & 0.720$\pm$0.04 & 0.697$\pm$0.06 & \textbf{0.729}$\pm$0.06 & 0.721$\pm$0.05 & 0.723$\pm$0.05 & 0.719$\pm$0.04 & 0.715$\pm$0.04 \\
Folding 17 & 0.684$\pm$0.05 & 0.703$\pm$0.05 & 0.677$\pm$0.07 & 0.716$\pm$0.05 & 0.712$\pm$0.04 & 0.712$\pm$0.04 & \textbf{0.730}$\pm$0.03 & 0.720$\pm$0.04 \\
Folding 18 & 0.660$\pm$0.06 & 0.714$\pm$0.06 & 0.675$\pm$0.08 & 0.685$\pm$0.08 & 0.713$\pm$0.06 & 0.706$\pm$0.06 & 0.735$\pm$0.05 & \textbf{0.735}$\pm$0.05 \\
Folding 19 & 0.694$\pm$0.02 & 0.692$\pm$0.04 & \textbf{0.727}$\pm$0.05 & 0.713$\pm$0.04 & 0.719$\pm$0.02 & 0.703$\pm$0.03 & 0.725$\pm$0.05 & 0.715$\pm$0.05 \\
\hline
Averaged  & 0.675$\pm$0.06 & 0.705$\pm$0.06 & 0.698$\pm$0.06 & 0.705$\pm$0.06 & 0.707$\pm$0.06 & 0.708$\pm$0.06 & 0.714$\pm$0.06 & \textbf{0.715}$\pm$0.06 \\
\hline
\end{tabular}
\label{table:auc_score}
\end{table*}

\textbf{Results from three real-world images.\quad} In order to further demonstrate the merit of \textsc{GraphStoIHT}, we compare it with another three popular methods: \textsc{NIHT}~\cite{blumensath2010normalized}, \textsc{CoSaMP}~\cite{needell2009cosamp}, \textsc{GraphCoSaMP}~\cite{hegde2015nearly}. We test all of them on three $50\times 50$ resized real images: Background, Angio, and Text provided in~\citet{hegde2015nearly}, where the first two images have one connected component while Text has four. \textsc{NIHT} and \textsc{CoSaMP} have sparsity $s$ as input parameter. \textsc{GraphStoIHT} shares the same head and tail projection as \textsc{GraphCoSaMP}. The learning rates $\eta$ of \textsc{IHT}, \textsc{StoIHT}, \textsc{GraphIHT} and \textsc{GraphStoIHT} are tuned from the set $\{0.2, 0.4, 0.6, 0.8\}$, and the block sizes $b$ of \textsc{StoIHT} and \textsc{GraphStoIHT} are tuned from the set $\{m/5, m/10 \}$. We tune $b$ and $\eta$ on an additional validation dataset with 100 observations. To clarify, the design matrix ${\bm A}$ used here is Gaussian matrix, different from the Fourier Matrix used in~\citet{hegde2015nearly}. Our method outperforms the others consistently.

%%%%%%%%%%%%%%%%%%%%%%%%%%%%%%%%%%%%%%%%%%%%%%%%%%%%%%%%%%%%%%%%%%%%%%%%%%%%%%%%
\subsection{Graph sparse logistic regression}
To further test our method on real-world dataset, we apply it to the breast cancer dataset in~\citet{van2002gene}, which contains 295 training samples including 78 positives (metastatic) and 217 negatives (non-metastatic). Each ${\bm a}_i$ in~(\ref{objective:logistic_regression}) is the training sample of patient $i$ with dimension $p=8,141$ (genes). Label $y_i=1$ if patient $i$ has metastatic growth; otherwise $-1$. We use the Protein-Protein Interaction (PPI) network in~\citet{jacob2009group}\footnote{This network was originally proposed by~\citet{chuang2007network}.}. There are 637 pathways in this PPI network. We restrict our analysis on 3,243 genes (nodes) which form a connected graph with 19,938 edges. Due to lack of edge weights, all weights have been set to 1.0. We fold the dataset uniformly into 5 subfolds as done by~\citet{jenatton2011structured} to make comparison later. All related parameters are tuned by 5-fold-cross-validation on each training dataset.  More experiment details are available in the Supplementary Material. We repeat the folding strategy 20 times, and Table~\ref{table:auc_score} reports AUC scores with standard deviation.

\textbf{Metastasis classification.\quad}  We compare our algorithm with the three aforementioned non-convex based methods and four $\ell^1/\ell^2$ mixed norm-based algorithms, $\ell^1$-\textsc{Pathway}, $\ell^1/\ell^2$-\textsc{Pathway}, $\ell^1$-\textsc{Edge}, and $\ell^1/\ell^2$-\textsc{Edge}\footnote{The code of these four methods is sourced from \url{http://cbio.ensmp.fr/~ljacob/documents/overlasso-package.tgz}.}. $\ell^1$-\textsc{Pathway} and $\ell^1/\ell^2$-\textsc{Pathway} use pathways as groups while $\ell^1$-\textsc{Edge} and $\ell^1/\ell^2$-\textsc{Edge} use edges as groups. On average, \textsc{GraphStoIHT} achieves 0.715 AUC score, the highest among the eight methods. 

\textbf{Gene identification.\quad} We also investigate different numbers of breast cancer-related genes identified by each method. Out of 25 breast cancer-related genes, \textsc{GraphStoIHT} finds 24\% of them, more than \textsc{GraphIHT} (20\%), \textsc{StoIHT} (16\%) and \textsc{IHT} (8\%).  The mixed norm-based methods $\ell^1$-\textsc{Pathway}, $\ell^1/\ell^2$-\textsc{Pathway}, $\ell^1$-\textsc{Edge}, and $\ell^1/\ell^2$-\textsc{Edge} find 16\%, 8\%, 12\%, and 12\% respectively. Our method outperforms other methods in finding cancer-related genes. 

%%%%%%%%%%%%%%%%%%%%%%%%%%%%%%%%%%%%%%%%%%%%%%%%%%%%%%%%%%%%%%%%%%%%%%%%%%%%%%%%
\section{Conclusion and Future Work}
\label{section:discussion}
In this paper, we have proposed \textsc{GraphStoIHT}, a stochastic gradient-based method for solving graph-structured sparsity constraint problems. We proved that it enjoys a linear convergence property. Experimental evaluation shows our method consistently outperforms other algorithms for both graph sparse linear regression and graph logistic regression on real-world datasets. In future, it would be interesting to see if one can apply the variance reduction techniques such as SAGA~\cite{defazio2014saga} and SVRG~\cite{johnson2013accelerating} to \textsc{GraphStoIHT}.

%%%%%%%%%%%%%%%%%%%%%%%%%%%%%%%%%%%%%%%%%%%%%%%%%%%%%%%%%%%%%%%%%%%%%%%%%%%%%%%%
\section*{Acknowledgement}
The authors would like to thank the five anonymous reviewers for their helpful comments on the paper. The work of Yiming Ying is supported by the National Science Foundation (NSF) under Grant No \#1816227. Baojian Zhou would like to thank Lin Xiao for her valuable proofreading on a draft of the paper.

\bibliography{references}
\bibliographystyle{icml2019}
%%%%%%%%%%%%%%%%%%%%%%%%%%%%%%%%   END    %%%%%%%%%%%%%%%%%%%%%%%%%%%%%%%%%%%%%%
%%%%%%%%%%%%%%%%%%%%%%%%%%%%%%%%%%%%%%%%%%%%%%%%%%%%%%%%%%%%%%%%%%%%%%%%%%%%%%%%

\onecolumn

\appendix
\section{Proofs}
\label{section:proofs}
%%%%%%%%%%%%%%%%%%%%%%%%%%%%%%%%%%%%%%%%%%%%%%%%%%%%%%%%%%%%%%%%%%%%%%%%%%%%%%%%
Before proving Lemma~\ref{lemma:lemma1}, we first prove Lemma~\ref{lemma:lemma2}, \ref{lemma:lemma3}, \ref{lemma:lemma4}, \ref{lemma:lemma5}, \ref{lemma:lemma6}. We prove the main Theorem~\ref{theorem:theorem_01} and three Corollaries after proving Lemma~\ref{lemma:lemma1}. In the following proofs, given any $\mathbb{M}$, $\overline{\mathbb{M}}$ 
is the closure of $\mathbb{M}$ under taking subsets, which is defined as $\overline{\mathbb{M}}=\{\Omega: \Omega \subseteq S, \text{ for some } S \in \mathbb{M}\}$. 

\begin{lemma}[co-coercivity]
If $f$ is convex differentiable and satisfies the $(\alpha,
\beta,\mathcal{M}(\mathbb{M}))$-RSS/RSC property, then we have 
\begin{align}
\| \nabla_\Omega f({\bm x}) - \nabla_\Omega f({\bm y}) \|^2 &\leq \beta
\langle {\bm x} - {\bm y}, \nabla f({\bm x}) - \nabla f({\bm y}) \rangle, 
\label{lemma:lemma2:inequ01}
\end{align}
where $\| {\bm x} \|_0 \cup \| {\bm y} \|_0 \subseteq \Omega$ and $ \Omega \in \overline{\mathbb{M}}$.
\label{lemma:lemma2}
\end{lemma}
%%%%%%%%%%%%%%%%%%%%%%%%%%%%%%%%%%%%%%%%%%%%%%%%%%%%%%%%%%%%%%%%%%%%%%%%%%%%%%%%
\begin{proof}
By the definition of $(\alpha,\beta,\mathcal{M}(\mathbb{M}))$-RSS/RSC property,
 we have
\begin{equation*}
B_{f}({\bm x}, {\bm y}) := f({\bm x}) - f({\bm y}) - 
\langle \nabla f({\bm y}), {\bm x} -{\bm y} \rangle 
\leq \frac{\beta}{2} \| {\bm x} - {\bm y}\|^2. 
\end{equation*}
Let ${\bm x}_0 \in \mathcal{M}(\mathbb{M})$ be any fixed vector. We define a 
surrogate function $\varphi({\bm y}) := f({\bm y}) - 
\langle \nabla f({\bm x}_0), {\bm y} \rangle$. 
$\varphi$ is also $\beta$-smoothness in $\mathcal{M}(\mathbb{M})$. 
To see this, notice that
\begin{align*}
B_\varphi({\bm x}, {\bm y}) &= \varphi({\bm x}) - \varphi({\bm y}) - 
\langle \nabla \varphi({\bm y}),{\bm x} - {\bm y} \rangle \\
&= f({\bm x}) - f({\bm y}) - 
\langle \nabla f({\bm x}_0),{\bm x} - {\bm y} \rangle 
 - \langle \nabla f({\bm y}) - \nabla f({\bm x}_0), 
 {\bm x} - {\bm y} \rangle \nonumber\\
&= f({\bm x}) - f({\bm y}) - \langle \nabla f({\bm y}), 
{\bm x} - {\bm y} \rangle \\
&= B_{f}({\bm x}, {\bm y})  \leq \frac{\beta}{2} \| {\bm x} -{\bm y}\|^2.
\end{align*}
%%%%%%%%%%%%%%%%%%%%%%%%%%%%%%%%%%%%%%%%%%%%%%%%%%%%%%%%%%%%%%%%%%%%%%%%%%%%%%%%
Since $\nabla \varphi({\bm x}_0) = {\bm 0}$, $\varphi$ gets its minimum at 
${\bm x}_0$ . By using the above inequality 
on $\varphi$ with ${\bm x}$ replaced by ${\bm y} - 1/\beta
\nabla_{\Omega} \varphi({\bm y})$ (noticing that 
$\|{\bm y} - 1/\beta
\nabla_{\Omega} \varphi({\bm y})\|_0 \subseteq \Omega$), we get
\begin{align*}
&\varphi({\bm x}) - \varphi({\bm y}) - 
\langle \nabla \varphi({\bm y}),{\bm x} - {\bm y} \rangle
\leq \frac{\beta}{2} \| {\bm x} -{\bm y} \|^2 \\
&\Leftrightarrow \varphi({{\bm y} - \frac{1}{\beta} \nabla_{\Omega} 
\varphi ({\bm y})}) - \varphi({\bm y}) - 
\langle \nabla \varphi({\bm y}), { - \frac{1}{\beta} 
\nabla_{\Omega} \varphi ({\bm y})} \rangle \leq \frac{\beta}{2} \| - 
\frac{1}{\beta} \nabla_{\Omega} \varphi ({\bm y}) \|^2 \\
&\Leftrightarrow \varphi({\bm x}_0)\leq \varphi ({\bm y} - 
\frac{1}{\beta} \nabla_{\Omega} \varphi({\bm y})) 
\leq \varphi({\bm y}) - \frac{1}{2 \beta} 
\| \nabla_{\Omega} \varphi({\bm y}) \|^2.
\end{align*}
By using the definition of $\varphi$ on the above inequality, we finally have
\begin{align}
&f({\bm x}_0) - \langle \nabla f ({\bm x}_0),{\bm x}_0 
\rangle \leq f({\bm y}) - \langle \nabla f({\bm x}_0 ), {\bm y} 
\rangle - \frac{1}{2\beta} \| \nabla_{\Omega} f
({\bm y}) - \nabla_{\Omega} f({\bm x}_0) \|^2 \nonumber \\
&\Leftrightarrow  f ({\bm x}_0) + \langle \nabla f ({\bm x}_0),{\bm y} - 
{\bm x}_0 \rangle  + \frac{1}{2\beta} \| \nabla_{\Omega} f
({\bm y}) - \nabla_{\Omega} f ({\bm x}_0) \|^2 \leq 
f ({\bm y}) \nonumber \\
&\Leftrightarrow \frac{1}{2\beta} \| \nabla_{\Omega} f ({\bm y}) - 
\nabla_{\Omega} f({\bm x}_0) \|^2 \leq f ({\bm y}) - 
f ({\bm x}_0) + \langle \nabla f ({\bm x}_0),{\bm x}_0 -
 {\bm y} \rangle \tag{*}
\end{align}
Let ${\bm x}_0 = {\bm x}$. By adding two copies of
~(*) with ${\bm x}$ and ${\bm y}$ 
interchanged, we prove the lemma.
\end{proof}
\begin{remark}
The proof of Lemma~\ref{lemma:lemma2} mainly follows~\citet{nesterov2013introductory} where the function considered is convex differentiable and has a Lipschitz continuous gradient in $\mathbb{R}^p$. Instead of proving the property in $\mathbb{R}^p$, we prove it in the subspace model $\mathcal{M}(\mathbb{M})$. Similar uses of this property can be found in~\citet{nguyen2017linear,shen2017tight}. One should notice that this co-coercivity property does not depend on the strong-convexity parameter $\alpha$, i.e., $\alpha \geq 0$.
\end{remark}

\begin{lemma}
For any convex differentiable function $f$ satisfies $(\alpha,\beta,\mathcal{M}(\mathbb{M}))$-RSS/RSC property with strongly convex parameter $\alpha$ and strongly smoothness parameter $\beta$, we have the following inequalities
\begin{align}
\alpha \| {\bm x} - {\bm y} \| 
&\leq \| \nabla_{\Omega} f({\bm x}) - \nabla_{\Omega} f({\bm y}) \|
\leq \beta \|{\bm x} - {\bm y}\| 
\label{lemma:lemma3:inequ01}
\end{align}
where $\|{\bm x} \|_0 \cup \|{\bm y}\|_0 \subseteq \Omega$ and 
$\Omega \in \overline{\mathbb{M}}$.
\label{lemma:lemma3}
\end{lemma}

%%%%%%%%%%%%%%%%%%%%%%%%%%%%%%%%%%%%%%%%%%%%%%%%%%%%%%%%%%%%%%%%%%%%%%%%%%%%%%%%
\begin{proof}
Since $f({\bm x})$ satisfies $(\alpha,\beta,\mathcal{M}(\mathbb{M}))$-RSS/RSC 
property, by using left inequality of~(\ref{inequ:rssc}) and 
exchanging ${\bm x }, {\bm y}$ and summing them together, we can get
\begin{align*}
\alpha \|{\bm x} -{\bm y} \|^2 &\leq \langle \nabla f({\bm x}) - 
\nabla f({\bm y}), {\bm x} - {\bm y} \rangle \\
&= \langle \nabla_{\Omega} f({\bm x}) - 
\nabla_{\Omega} f({\bm y}), {\bm x} - {\bm y} \rangle \tag{*}\\
&\leq \| \nabla_{\Omega}f({\bm x}) - \nabla_{\Omega} f({\bm y}) 
\| \cdot \| {\bm x} - {\bm y} \|,
\end{align*}
where the first equality follows by $\|{\bm x}\|_0 \cup \|{\bm y}\|_0
 \subseteq \Omega$ and the second inequality follows by the 
Cauchy-Schwarz inequality. Therefore, the LHS of inequality
~(\ref{lemma:lemma2:inequ01}) is obtained by eliminating 
$\|{\bm x} - {\bm y}\|$ on both sides(Notice that when ${\bm x} = {\bm y}$, 
two inequalities stated in this lemma are trivially true). Furthermore, 
$\| \nabla_{\Omega} f({\bm x}) - \nabla_{\Omega} f({\bm y}) \|^2$ can be 
upper bounded as
\begin{align*}
\| \nabla_{\Omega} f({\bm x}) - \nabla_{\Omega} f ({\bm y}) 
\|^2 &= \| \nabla_{\Omega} f({\bm x}) - \nabla_{\Omega} f({\bm y}) \|^2 \\
&\leq \beta \langle \nabla f({\bm x}) - \nabla f({\bm y}), 
{\bm x} -{\bm y} \rangle \\
&= \beta  \langle \nabla_\Omega f({\bm x}) - \nabla_\Omega f({\bm y}), 
{\bm x} - {\bm y} \rangle \\
&\leq \beta \| \nabla_{\Omega} f({\bm x}) - \nabla_{\Omega} f({\bm y}) 
\|\cdot \| {\bm x} - {\bm y} \|,
\end{align*}
where the first inequality follows by Lemma~\ref{lemma:lemma2}.
Hence, the RHS of~(\ref{lemma:lemma2:inequ01}) is obtained by eliminating 
$\| \nabla_{\Omega} f({\bm x}) - \nabla_{\Omega } f({\bm y}) \|$ 
on both sides. We prove the lemma.
\end{proof}

%%%%%%%%%%%%%%%%%%%%%%%%%%%%%%%%%%%%%%%%%%%%%%%%%%%%%%%%%%%%%%%%%%%%%%%%%%%%%%%%
\begin{lemma}[~\citet{nguyen2017linear}]
Let $\xi_t$ be a discrete random variable defined on $[n]$ and its probability
mass function is defined as $Pr(\xi_t = i):=1/n$, which means the probability
of $\xi_t$ selects $i$th block at time $t$. For any fixed sparse vectors 
${\bm x}$, ${\bm y}$ and $0<\tau<\frac{2}{\beta}$,
we have
\begin{align}
\mathbb{E}_{\xi_t} \Big\| {\bm x} - {\bm y} - \tau\Big(\nabla_{\Omega} 
f_{\xi_t}({\bm x}) - \nabla_{\Omega}f_{\xi_t}({\bm y})\Big)\Big\|
&\leq \sqrt{\alpha\beta\tau^2 - 2\alpha\tau + 1} \| {\bm x} - {\bm y} \|, 
\label{lemma:lemma4_equ01}
\end{align}
where $\Omega$ be such set that 
$\|{\bm x}\|_0 \cup \|{\bm y}\|_0 \subseteq \Omega$ and 
$\Omega \in \overline{\mathbb{M}}$. 
\label{lemma:lemma4}
\end{lemma}
%%%%%%%%%%%%%%%%%%%%%%%%%%%%%%%%%%%%%%%%%%%%%%%%%%%%%%%%%%%%%%%%%%%%%%%%%%%%%%%%
\begin{proof} We first try to obtain an upper bound  
$\mathbb{E}_{\xi_t} \| {\bm x} - {\bm y} - \tau
(\nabla_\Omega f_{\xi_t}({\bm x}) - \nabla_\Omega f_{\xi_t}({\bm y}))\|^2$ as 
the following
\begin{align*}
&\mathbb{E}_{\xi_t} \Big\| {\bm x} - {\bm y} - \tau
\Big(\nabla_\Omega f_{\xi_t}({\bm x}) - \nabla_\Omega f_{\xi_t}({\bm y})\Big) 
\Big\|^2 \\
&= \| {\bm x} - {\bm y} \|^2 - 2\tau \mathbb{E}_{\xi_t} \Big\langle {\bm x} - 
{\bm y}, \nabla_\Omega f_{\xi_t}({\bm x}) - \nabla_\Omega 
 f_{\xi_t}({\bm y}) \Big\rangle + \tau^2 \mathbb{E}_{\xi_t}
 \Big\| \nabla_{\Omega} f_{\xi_t}({\bm x}) - 
 \nabla_{\Omega} f_{\xi_t}({\bm y}) \Big\|^2 \\
&= \| {\bm x} - {\bm y} \|^2 - 2\tau\Big\langle {\bm x} - {\bm y}, 
\mathbb{E}_{\xi_t}\Big[\nabla f_{\xi_t}({\bm x}) -  
\nabla f_{\xi_t}({\bm y})\Big]\Big\rangle + \tau^2 \mathbb{E}_{\xi_t} 
\| \nabla_{\Omega} f_{\xi_t}({\bm x}) - 
\nabla_{\Omega} f_{\xi_t}({\bm y}) \|^2 \\
&\leq \| {\bm x} - {\bm y} \|^2 - 2\tau\Big\langle {\bm x} - {\bm y}, 
\mathbb{E}_{\xi_t}\Big[\nabla f_{\xi_t}({\bm x}) -  
\nabla f_{\xi_t}({\bm y})\Big]\Big\rangle + \tau^2\beta
\Big\langle {\bm x} - {\bm y}, \mathbb{E}_{\xi_t}\Big[
\nabla f_{\xi_t}({\bm x}) - \nabla f_{\xi_t}({\bm y})\Big]\Big\rangle \\
&= \| {\bm x} - {\bm y} \|^2 + (\tau^2\beta- 2\tau)
\Big\langle {\bm x} - {\bm y}, F(\bm x) - F(\bm y)\Big\rangle \\
&\leq (\alpha\beta\tau^2 - 2\alpha\tau + 1)\| {\bm x} - {\bm y} \|^2,
\end{align*}
where the second equality uses the fact that $\|\bm x \|_0 \cup \|\bm y\|_0 \subseteq \Omega$, the first inequality follows from Lemma~\ref{lemma:lemma2}, the third equality is obtained by using the fact that $\mathbb{E}_{\xi_t}[\nabla f_{\xi_t}({\bm x}) - \nabla f_{\xi_t}({\bm y})]
=\nabla F({\bm x}) - \nabla F({\bm y})$ and the last inequality is due to the restricted strong convexity of $F({\bm x})$ on $\mathcal{M}(\mathbb{M})$. We complete the proof by taking the square root of both sides and using the fact: for any random variable $X$, we have $(\mathbb{E}[X])^2 \leq \mathbb{E}X^2$.
\end{proof}
%%%%%%%%%%%%%%%%%%%%%%%%%%%%%%%%%%%%%%%%%%%%%%%%%%%%%%%%%%%%%%%%%%%%%%%%%%%%%%%%
\begin{lemma}[Matrix Chernoff~(\citet{tropp2012user})]
Consider a finite sequence {${\bm M}_1, {\bm M}_2,\cdots, {\bm M}_r$} of $d \times d$, independent, random, self-adjoint matrices that satisfy
$ {\bm M}_i \succeq 0 \text{ and } \lambda_{max} ({\bm M}_i ) \leq R$ almost surely. Let $\theta_{max} = \lambda_{max} ( \sum_{i=1}^{r} \mathbb{E}[{\bm M}_i])$. Then for $\nu \geq 0$,
\begin{equation}
Pr \Big\{ \lambda_{max} \Big( \sum_{i=1}^{r} {{\bm M}_i} \Big) \geq (1+\nu) \theta_{max} \Big\} \leq d e^{\frac{\theta_{max}}{R
} h(\nu)}, \label{inequ:27}
\end{equation}
where $h(\nu) = \nu - (1+\nu) \log (1+\nu)$.
\label{lemma:lemma5}
\end{lemma}

\begin{lemma} Given the matrix $\mathbf{A}_{B_i}^\mathsf{T} \mathbf{A}_{B_i} = \sum_{j=1}^{m/n} {\bm a}_{i_j}{\bm a}_{i_j}^\mathsf{T}$ and $\|{\bm a}_{i_j}\|=1$. ${\bm a}_{i_1},{\bm a}_{i_2},\ldots$ are independent and random. The matrix norm of $\mathbf{A}_{B_i}^\mathsf{T} \mathbf{A}_{B_i}$ can be bounded as $ \|\mathbf{A}_{B_i}^\mathsf{T} \mathbf{A}_{B_i} \|_2 < (1+\nu) \theta_{max}$ with probability $1-p \exp{(-\theta_{max}\nu/4)}$, where $\theta_{max} = \lambda_{max} (\sum_{j=1}^{n/m} \mathbb{E}[{\bm a}_{i_j}{\bm a}_{i_j}^\mathsf{T}])$, and $\nu \geq 1$.
\label{lemma:lemma6}
\end{lemma}

\begin{proof}
Let ${\bm M}_j = {\bm a}_{i_j}{\bm a}_{i_j}^\mathsf{T}$. Then the finite 
sequence ${\bm M}_1,{\bm M}_2,\ldots,{\bm M}_{m/n}$ is independent, 
random, and self-adjoint matrices. We write the matrix as a summation 
as the following
\begin{equation}
\mathbf{A}_{B_i}^\mathsf{T} \mathbf{A}_{B_i} = \sum_{j=1}^{m/n} {\bm a}_{i_j}{\bm a}_{i_j}^\mathsf{T}  = \sum_{j=1}^{m/n} {\mathbf{M}_j}. \nonumber
\end{equation}
Given $\nu \geq 1$, we know the fact that $\nu - (1+\nu) \log (1+\nu) \leq - \frac{\nu \log (1+\nu)}{2}$. Furthermore, $-\log(1+\nu) \leq - \frac{\nu}{1+\nu}$, which means $\nu - (1+\nu) \log (1+\nu)  \leq \frac{-\nu^2}{2(1+\nu)}$. As $\frac{\nu}{1+\nu} \geq \frac{1}{2}$, $\nu - (1+\nu) \log (1+\nu) \leq - \frac{\nu}{4}$. Therefore, we use inequality~(\ref{inequ:27}) of Lemma~\ref{lemma:lemma5} by replacing $r$ with $m/n$ and $R$ with $1$ ( by noticing that ${\bm a}_{i_j}{\bm a}_{i_j}^\mathsf{T}$ is normalized).
\begin{align*}  
&Pr \Big\{ \lambda_{max} \Big( \sum_{j=1}^{m/n} {\mathbf{M}_j} \Big) \geq (1+\nu) \theta_{max} \Big\} \\
&\leq p \exp{(\frac{\theta_{max}}{R
} ( \nu - (1+\nu) \log (1+\nu)))} \\
&\leq p \exp{(\frac{\theta_{max}}{R} \frac{-\nu}{4})} = p \exp{(-\frac{\theta_{max} \nu}{4})}.
\end{align*}
In other words, the probability of $Pr \Big\{ \lambda_{max} \Big( \sum_{j=1}^{m/n} {\mathbf{M}_j} \Big) < (1+\nu) \theta_{max} \Big\}$ is greater than $1- p \exp{(-\frac{\theta_{max} \nu}{4})}$. We prove the lemma.
\end{proof}

\begin{remark}
Lemma~\ref{lemma:lemma6} shows that the maximum eigenvalue of matrix ${\bm A}_{B_i}^\top {\bm A}_{B_i}$ can be upper bounded properly. Our Corollary~\ref{corollary:corollary2} depends on Lemma~\ref{lemma:lemma6}.
\end{remark}

\addtocounter{lemma}{-6}
\begin{lemma}
Denote  $\xi_{[t]} = (\xi_0,\xi_1,\ldots,\xi_t)$ as the history of the 
stochastic process $\xi_0,\xi_1,\ldots,$ up to time $t$ and 
all random variables $\xi_t$ are independent of each other. 
Define the probability mass function $Pr(\xi_t=i)=1/n, 1\leq i \leq n$. 
Given an optimal solution ${\bm x}^* \in \mathcal{M}(\mathbb{M})$, head
projection model $(c_\mathcal{H}, \mathbb{M}\oplus \mathbb{M}_\mathcal{T}, 
\mathbb{M}_\mathcal{H})$, tail projection model 
$(c_\mathcal{T}, \mathbb{M}, \mathbb{M}_\mathcal{T})$ and $f_{\xi_t}(\cdot)$ 
and $F({\bm x})$ satisfies Assumption~\ref{assumption:assumption_01}, then we have
\begin{equation}
\mathbb{E}_{\xi_t}\|({\bm x}^t - {\bm x}^*)_{H^c}\|  
\leq \sqrt{1 - \alpha_0^2} \mathbb{E}_{\xi_t}\|{\bm x}^t - {\bm x}^*\| 
+ \Big( \frac{\beta_0}{\alpha_0} + \frac{\alpha_0 \beta_0}{\sqrt{1-\alpha_0^2}} 
\Big)\mathbb{E}_{\xi_t}\|\nabla_I f_{\xi_t}({\bm x}^*)\|, 
\end{equation}
where 
\begin{align*}
H &:= \|{\rm P}(\nabla f_{\xi_t}({\bm x}^t), 
\mathbb{M} \oplus \mathbb{M}_\mathcal{T}, 
\mathbb{M}_{\mathcal{H}} )\|_0,\quad\alpha_0 := c_\mathcal{H} \alpha \tau - 
\sqrt{\alpha\beta \tau^2 - 2\alpha\tau + 1},\quad\beta_0 := (1+c_\mathcal{H})\tau,\\
I &:= \argmax_{S \in \mathbb{M} \oplus 
\mathbb{M}_\mathcal{T} \oplus \mathbb{M}_\mathcal{H} } 
\mathbb{E}_{\xi_t}\|\nabla_S f_{\xi_t}({\bm x}^*)\|, \text{ and } 
0 < \tau < 2/\beta.
\end{align*}
\end{lemma}
%%%%%%%%%%%%%%%%%%%%%%%%%%%%%%%%%%%%%%%%%%%%%%%%%%%%%%%%%%%%%%%%%%%%%%%%%%%%%%%%
\begin{proof}
Let ${\bm r}^t = {\bm x}^t - {\bm x}^*$ and $\Omega 
:= \|{\bm r}^t\|_0 \in \overline{\mathbb{M}\oplus \mathbb{M}_\mathcal{T}}$. 
The component $\mathbb{E}_{\xi_t}\|\nabla_H f_{\xi_t}({\bm x}^t)\|$ 
can be lower bounded as
\begin{align*}
\mathbb{E}_{\xi_t} \|\nabla_H f_{\xi_t}({\bm x}^t)\| 
&\geq c_\mathcal{H} \mathbb{E}_{\xi_t} \max_{S \in \mathbb{M} 
\oplus \mathbb{M}_\mathcal{T}} \| \nabla_S f_{\xi_t}({\bm x}^t) \| \\
&= c_\mathcal{H} \mathbb{E}_{\xi_t}
\max_{ {S'} \in \overline{\mathbb{M} \oplus 
\mathbb{M}_\mathcal{T}}} \| \nabla_{S'} f_{\xi_t}({\bm x}^t) \| \\
&\ge c_\mathcal{H} \mathbb{E}_{\xi_t} \| \nabla_\Omega 
f_{\xi_t}({\bm x}^t) \| \\
&\ge c_\mathcal{H} \mathbb{E}_{\xi_t}(\| \nabla_\Omega f_{\xi_t}({\bm x}^t)- 
\nabla_\Omega f_{\xi_t}({\bm x}^*) \|  - \|\nabla_\Omega f_{\xi_t}({\bm x}^*)\|_2 )\\
&\ge c_\mathcal{H}(\| \nabla_\Omega F({\bm x}^t)- \nabla_\Omega F({\bm x}^*) \|  - 
\mathbb{E}_{\xi_t|\xi_{[t-1]}}\|\nabla_\Omega f_{\xi_t}({\bm x}^*)\|_2 )\\
&\ge c_\mathcal{H} \alpha  \|{\bm r}^t\|_2 - 
c_\mathcal{H} \mathbb{E}_{\xi_t}\|\nabla_I f_{\xi_t}({\bm x}^*)\|, 
\end{align*}
where the first inequality follows by the definition of the head projection
 and the last inequality follows by~(\ref{lemma:lemma2:inequ01}) 
using $(\alpha,\beta,\mathbb{M}\oplus \mathbb{M}_\mathcal{T} \oplus 
\mathbb{M}_\mathcal{H})$-RSC/RSS property. The component $\|\nabla_H 
f_{\xi_t}({\bm x}^t)\|$ can also  be upper bounded as
%%%%%%%%%%%%%%%%%%%%%%%%%%%%%%%%%%%%%%%%%%%%%%%%%%%%%%%%%%%%%%%%%%%%%%%%%%%%%%%%
\begin{align*}
\mathbb{E}_{\xi_t}\|\nabla_H f_{\xi_t}({\bm x}^t)\| 
&\leq \frac{1}{\tau} \mathbb{E}_{\xi_t}\Big\| \tau 
\Big( \nabla_H f_{\xi_t}({\bm x}^t)-  \nabla_H f_{\xi_t}({\bm x}^*) \Big) \Big\| 
+ \mathbb{E}_{\xi_t}\|\nabla_H f_{\xi_t}({\bm x}^*)\| \\
&\leq \frac{1}{\tau} \mathbb{E}_{\xi_t}
\Big\| \tau \Big(\nabla_H f_{\xi_t}({\bm x}^t)-  
\nabla_H f_{\xi_t}({\bm x}^*)\Big)  - {\bm r}_H^t 
\Big\|_2 +  \frac{1}{\tau} \mathbb{E}_{\xi_t}
 \| {\bm r}_H^t \| + \mathbb{E}_{\xi_t}\|\nabla_H f_{\xi_t}({\bm x}^*)\| \\
&\leq \frac{1}{\tau} \mathbb{E}_{\xi_t}\Big\| 
\tau \Big( \nabla_{H\cup \Omega} 
f_{\xi_t}({\bm x}^t) - \nabla_{H\cup \Omega} f_{\xi_t}({\bm x}^*) \Big) - 
{\bm r}_{H \cup \Omega}^t  \Big\|_2 + \frac{1}{\tau} \mathbb{E}_{\xi_t}\|{\bm r}_H^t \|
+ \mathbb{E}_{\xi_t}\|\nabla_H f_{\xi_t}({\bm x}^*)\| \\
&= \frac{1}{\tau} \mathbb{E}_{\xi_t}\Big\| 
{\bm r}^t  - \tau \Big( \nabla_{H\cup \Omega} 
f_{\xi_t}({\bm x}^t) - \nabla_{H\cup \Omega} f_{\xi_t}({\bm x}^*) \Big) \Big\| 
+ \frac{1}{\tau} \mathbb{E}_{\xi_t}\|{\bm r}_H^t \|  + 
\mathbb{E}_{\xi_t}\|\nabla_H f_{\xi_t}({\bm x}^*)\| \\
&\leq \frac{\sqrt{ \alpha\beta\tau^2 - 2 \alpha \tau + 1}}{\tau}   \|{\bm r}^t\| 
+ \frac{1}{\tau}   \| {\bm r}_H^t \| + \mathbb{E}_{\xi_t}\|\nabla_{I} f_{\xi_t}({\bm x}^*)\| ,
\end{align*}
where the last inequality follows 
from~(\ref{lemma:lemma4_equ01}) of Lemma~\ref{lemma:lemma4} by 
using the fact that $H\cup \Omega$ is in 
$\mathbb{M}\oplus\mathbb{M}_{\mathcal{T}}\oplus\mathbb{M}_{\mathcal{H}}$. 
Combining the two bounds, we obtain the inequality:
\begin{equation*}
\|{\bm r}^t_H\|  \ge \alpha_0 \|{\bm r}^t\| - \beta_0 \mathbb{E}_{\xi_t}\|\nabla_I f_{\xi_t}({\bm x}^*)\|,
\end{equation*}
%%%%%%%%%%%%%%%%%%%%%%%%%%%%%%%%%%%%%%%%%%%%%%%%%%%%%%%%%%%%%%%%%%%%%%%%%%%%%%%%
where $\alpha_0 = c_\mathcal{H} \alpha \tau - \sqrt{ \alpha\beta \tau^2 -
 2 \alpha\tau + 1}$ and $\beta_0 = (1+c_\mathcal{H}) \tau$. 
 In order to obtain an upper bound of $\|{\bm r}_{H^c}^t \|$, we assume ${\bm r}^t \ne {\bm 0}$. Otherwise, our statement is trivially true. Two possible cases exist. The first case is that if $ \alpha_0 \|{\bm r}^t\| - \beta_0 \mathbb{E}_{\xi_t}\|\nabla_I f_{\xi_t}({\bm x}^*)\| \leq 0$, i.e. $\alpha_0 \|{\bm r}^t\| \leq  \beta_0 \mathbb{E}_{\xi_t}\|\nabla_I f_{\xi_t}({\bm x}^*)\|$, then we have $\|{\bm r}^t_{H^c}\| \le \| {\bm r}^t \| \leq  \frac{\beta_0}{\alpha_0} \mathbb{E}_{\xi_t}\|\nabla_I f_{\xi_t}({\bm x}^*)\|$. The second case is that if $\alpha_0 \|{\bm r}^t\| - \beta_0 \mathbb{E}_{\xi_t}\|\nabla_I f_{\xi_t}({\bm x}^*)\|>0$, i.e. $\alpha_0  \|{\bm r}^t\| >  \beta_0 \mathbb{E}_{\xi_t}\|\nabla_I f_{\xi_t}({\bm x}^*)\|$, then we have 
\begin{equation*}
\|{\bm r}^t_H\| \ge \left(\alpha_0 - \frac{\beta_0 \mathbb{E}_{\xi_t}\|\nabla_I f_{\xi_t}({\bm x}^*)\| 
}{\|{\bm r}^t\|} \right)  \|{\bm r}^t\|.
\end{equation*}
Moreover, notice that $\|{\bm r}_{H}^t\|^2 = \|{\bm r}^t\|^2 - 
\|{\bm r}^t_{H^c}\|_2^2$. We obtain 
\begin{align*}
\|{\bm r}^t_{H^c}\| &\le \|{\bm r}^t\| \sqrt{1 - \left(\alpha_0 - 
\frac{\beta_0 \mathbb{E}_{\xi_t}\|\nabla_I f_{\xi_t}({\bm x}^*)\| }{\|{\bm r}^t\|} \right)^2}.     
\end{align*}
Denote $x_0 := \alpha_0 - \beta_0 \mathbb{E}_{\xi_t}\|\nabla_I f_{\xi_t}({\bm x}^*)\| / \|{\bm r}^t\|$. 
Notice that $x_0 \in [0,1)$. Define function $\varphi(x) = 1/\sqrt{1 - x^2} -
x x_0 / \sqrt{1-x^2}$ on $(0,1)$. The first derivative of $\varphi$ is
$\varphi'(x)=(x-x_0)/(1-x^2)\sqrt{1-x^2}$ and $\varphi$ gets its minimum
at $x_0$, i.e., $\varphi(x_0)\leq \varphi(x)$. Therefore, substituting 
into the bound for $\|{\bm r}^t_{H^c}\|$, we get 
%%%%%%%%%%%%%%%%%%%%%%%%%%%%%%%%%%%%%%%%%%%%%%%%%%%%%%%%%%%%%%%%%%%%%%%%%%%%%%%%
\begin{align*}
\|{\bm r}^t_{H^c}\|  &\leq \|{\bm r}^t\| \varphi(x_0) \leq \|{\bm r}^t\| 
\varphi(x) \\
&= \|{\bm r}^t\| \left(\frac{1}{\sqrt{1 - x^2}} - \frac{x}{\sqrt{1-x^2}} 
\left(\alpha_0 - \frac{\beta_0 \mathbb{E}_{\xi_t}\|\nabla_I f_{\xi_t}({\bm x}^*)\|}{\|{\bm r}^t\|} 
\right)\right) \\
&= \frac{1 - x\alpha_0 }{\sqrt{1 - x^2}} \|{\bm r}^t\| + \frac{x\beta_0}
{\sqrt{1-x^2}} \mathbb{E}_{\xi_t}\|\nabla_I f_{\xi_t}({\bm x}^*)\|.
\end{align*}
In order to make the overall convergence rate as small as possible, we should 
make $(1 - x\alpha_0)/\sqrt{1 - x^2}$ as small as possible. Define the $\varphi$
function again as $\varphi(x) = (1 - x\alpha_0)/\sqrt{1 - x^2}$, we obtain 
the minimum  at $x=\alpha_0$.
Therefore, by combining the two cases, we obtain 
\begin{equation*}
\|{\bm r}^t_{H^c}\|  \le \sqrt{1 - \alpha_0^2} \|{\bm r}^t\| +
\left[\frac{\beta_0}{\alpha_0} + 
\frac{\alpha_0 \beta_0}{\sqrt{1-\alpha_0^2}}\right] 
\mathbb{E}_{\xi_t}\|\nabla_I f_{\xi_t} ({\bm x}^*)\|,
\end{equation*}
which proves the lemma.
\end{proof}
\begin{remark}
We should point out that we mainly follow the proof of Lemma 10 in~\citet{hegde2015approximation} where they consider the least square as the objective function with the model-RIP property. However, there are two main differences: 1) we assume a more general case, that is, our objective $F({\bm x})$ is strongly convex and each $f_{\xi_t}(\bm x)$ is strong smoothness. The above lemma can be applied to any function that satisfies this property; 2) The inequality~(\ref{inequ:7}) satisfies in stochastic setting.
\end{remark}

%%%%%%%%%%%%%%%%%%%%%%%%%%%%%%%%%%%%%%%%%%%%%%%%%%%%%%%%%%%%%%%%%%%%%%%%%%%%%%%%
\addtocounter{theorem}{-1}
\begin{theorem}
Denote  $\xi_{[t]} = (\xi_0,\xi_1,\ldots,\xi_t)$ as the history of the 
stochastic process $\xi_0,\xi_1,\ldots$ up to time $t$, and all random 
variables $\xi_t$ are independent of each other. Define the probability 
mass function $Pr(\xi_t=i)=1/n$. Let ${\bm x}^*$ be an optimal solution of~(\ref{equ:objective_function}) 
and ${\bm x}^0$ be the start point of Algorithm~
\ref{alg:graph-sto-iht}. If we choose a constant learning rate with
$\eta_t = \eta$, then the solution ${\bm x}^{t+1}$ generated 
by Algorithm~\ref{alg:graph-sto-iht} satisfies 
\begin{equation*}
\mathbb{E}_{\xi_{[t]}} \|{\bm x}^{t+1} - {\bm x}^* \| 
\leq \kappa^{t+1} \| {\bm x}^0 - {\bm x}^* \| + \frac{1}{1-\kappa} 
\Bigg(\frac{\beta_0}{\alpha_0} + \frac{\alpha_0 \beta_0}{\sqrt{1-\alpha_0^2}} 
+ \eta\Bigg) \mathbb{E}_{\xi_{t}}  \| \nabla_{I} f_{\xi_t} ({\bm x}^*) \|,
\end{equation*}
where 
\begin{align*}
\kappa &= (1 + c_\mathcal{T})\Big(\sqrt{ \alpha\beta\eta^2 - 2 \alpha\eta + 1} 
 + \sqrt{1 - \alpha_0^2}\Big),\quad\quad \alpha_0 = c_\mathcal{H} \alpha \tau - 
\sqrt{\alpha\beta \tau^2 - 2\alpha\tau + 1},\\
\beta_0 &= (1+c_\mathcal{H})\tau,
I = \argmax_{S \in \mathbb{M} \oplus 
\mathbb{M}_\mathcal{T} \oplus \mathbb{M}_\mathcal{H} } 
\mathbb{E}_{\xi_t}\|\nabla_S f_{\xi_t}({\bm x}^*)\|, \text{ and } \tau,\eta \in (0, 2/\beta).
\end{align*}

\end{theorem}
%%%%%%%%%%%%%%%%%%%%%%%%%%%%%%%%%%%%%%%%%%%%%%%%%%%%%%%%%%%%%%%%%%%%%%%%%%%%%%%%
\begin{proof}
Since ${\bm x}^{t+1}$ is completely determined by the realizations of 
the independent random variables $(\xi_0,\xi_1,\ldots, \xi_t)$, the total 
expectation of the approximation error $\| {\bm x}^{t+1} - {\bm x}\|$
can be taken as $\mathbb{E} \| {\bm x}^{t+1} - {\bm x}^*\| 
:= \mathbb{E}_{\xi_{[t]}}  \| {\bm x}^{t+1} - {\bm x}^*\|
= \mathbb{E}_{\xi_t|\xi_{[t-1]}} \| {\bm x}^{t+1} - {\bm x}\|$. Without loss 
of generality, $\mathbb{E}_{\xi_{[-1]}} \|{\bm x}^0 - {\bm x}^* \|^2 
=  \|{\bm x}^0 - {\bm x}^* \|^2$. Define ${\bm r}^{t+1} 
:= {\bm x}^{t+1} - {\bm x}^*, \Omega := \text{supp}({\bm r}^{t+1}),
H:=\|{\rm P} (\nabla f_{\xi_t} ({\bm x}^t), 
\mathbb{M}\oplus\mathbb{M}_{\mathcal{T}}, {\mathbb{M}_\mathcal{H}})\|_0,
T:=\|{\rm P} ( {\bm x}^{t} - \eta{\bm b}^t,
\mathbb{M},{\mathbb{M}_\mathcal{T}})\|_0$ 
as the support of head projection and tail projection respectively.
Firstly, we try to get an upper bound of $\|{\bm r}^{t+1}\|$ 
as the following
\begin{align*}
\mathbb{E}_{\xi_{[t]}}\|{\bm x}^{t+1} - {\bm x}^*\|
&= \mathbb{E}_{\xi_t| \xi_{[t-1]}}\|{\bm x}^{t+1} - {\bm x}^*\| \\
&= \mathbb{E}_{\xi_t| \xi_{[t-1]}}\| {\rm P}({\bm x}^t -
\eta {\bm b}^t,\mathbb{M},\mathbb{M}_\mathcal{T}) - {\bm x}^* \| \\
&\le \mathbb{E}_{\xi_t| \xi_{[t-1]}}\|{\rm P}({\bm x}^t - 
\eta {\bm b}^t,\mathbb{M},\mathbb{M}_\mathcal{T}) - 
({\bm x}^t - \eta {\bm b}^t)\| + \mathbb{E}_{\xi_t| \xi_{[t-1]}}
\| ({\bm x}^t - \eta {\bm b}^t) - {\bm x}^* \| \\
&\le (1 + c_\mathcal{T}) \mathbb{E}_{\xi_t| \xi_{[t-1]}}\| 
{\bm x}^t - \eta {\bm b}^t - {\bm x}^*\| \\
&= (1 + c_\mathcal{T}) \mathbb{E}_{\xi_t| \xi_{[t-1]}}\| {\bm x}^t 
- \eta\nabla_H f_{\xi_t} ({\bm x}^t) - {\bm x}^* \| \nonumber \\
&= (1 + c_\mathcal{T}) \mathbb{E}_{\xi_t| \xi_{[t-1]}}\|{\bm r}^t 
- \eta \nabla_H f_{\xi_t} ({\bm x}^t) \|,
\end{align*}
%%%%%%%%%%%%%%%%%%%%%%%%%%%%%%%%%%%%%%%%%%%%%%%%%%%%%%%%%%%%%%%%%%%%%%%%%%%%%%%%
where the first inequality follows by the triangle inequality and the 
second inequality follows by the definition of the tail projection. To be 
more specific, denote $\Gamma := \|{\bm x}^*\|_0 \in \mathbb{M}$ and given
any vector ${\bm a}$, we always have $\| {\bm a} - {\bm a}_\Gamma\| 
\leq \|  {\bm a} - {\bm x}_\Gamma^* \|$. By the definition of 
the tail projection, one can get $\|{\bm a}  - {\bm a}_T \| 
\leq c_\mathcal{T} \min_{S\in \mathbb{M}} \|{\bm a} - {\bm a}_S \| 
\leq c_\mathcal{T} \| {\bm a}^t - {\bm a}_\Gamma^t\|$. We get 
the second inequality by replacing ${\bm a}$ with ${\bm x}^t - \eta{\bm b}^t$. 
The third equality follows by Line 5 of Algorithm~\ref{alg:graph-sto-iht}. 
In the rest of the proof, we just need to bound the term, 
$\|{\bm r}^t - \eta \nabla_H f_{\xi_t}({\bm x}^t)\| $. 
Indeed, we can further bound it as
%%%%%%%%%%%%%%%%%%%%%%%%%%%%%%%%%%%%%%%%%%%%%%%%%%%%%%%%%%%%%%%%%%%%%%%%%%%%%%%%
\begin{align*}
&\mathbb{E}_{\xi_t| \xi_{[t-1]}}\|{\bm r}^t - 
\eta \nabla_H f_{\xi_t} ({\bm x}^t) \| 
=\mathbb{E}_{\xi_t| \xi_{[t-1]}}\|{\bm r}_{H^c}^t + {\bm r}_H^t - \eta
\nabla_H f_{\xi_t} ({\bm x}^t) \| \\
&\leq \mathbb{E}_{\xi_t| \xi_{[t-1]}}
\Bigg[\Big\| {\bm r}_H^t - \eta\Big( \nabla_H f_{\xi_t} ({\bm x}^t) - 
\nabla_H f_{\xi_t} ({\bm x}^*) \Big) \Big\|  + \eta \|
\nabla_H f_{\xi_t} ({\bm x}^*) \| + \|{\bm r}_{H^c}^t \| \Bigg]\\
&\leq \mathbb{E}_{\xi_t| \xi_{[t-1]}} \Bigg[\Big\| {\bm r}_{H \cup \Omega}^t
- \eta\Big( \nabla_{H \cup \Omega} f_{\xi_t} ({\bm x}^t) - 
\nabla_{H \cup \Omega} f_{\xi_t} ({\bm x}^*)\Big) \Big\|  
+ \eta  \| \nabla_H f_{\xi_t} ({\bm x}^*) \| + \|{\bm r}_{H^c}^t \| \Bigg]\\
&\leq \sqrt{ \alpha\beta\eta^2 - 2 \alpha\eta + 1} 
\mathbb{E}_{\xi_{[t-1]}}\|{\bm r}^t \|
+ \eta \mathbb{E}_{\xi_t}\| \nabla_{I} f_{\xi_t} ({\bm x}^*) \| + 
\mathbb{E}_{\xi_t|\xi_{[t-1]}}\|{\bm r}_{H^c}^t \| \\
&\leq \Bigg(\sqrt{ \alpha\beta\eta^2 - 2 \alpha\eta + 1} 
+ \sqrt{1 - \alpha_0^2} \Bigg) \mathbb{E}_{\xi_{[t-1]}}\|{\bm r}^t\|  + 
\Bigg(\frac{\beta_0}{\alpha_0} + 
\frac{\alpha_0 \beta_0}{\sqrt{1-\alpha_0^2}} + \eta\Bigg) 
\mathbb{E}_{\xi_t}\|\nabla_I f_{\xi_t}({\bm x}^*)\|,
\label{inequ:theorem_1-01}
\end{align*}
where the first three inequalities follow by the triangle inequality
and the fourth inequality uses Lemma~(\ref{lemma:lemma3}).  
Combine these two bounds together, we obtain
%%%%%%%%%%%%%%%%%%%%%%%%%%%%%%%%%%%%%%%%%%%%%%%%%%%%%%%%%%%%%%%%%%%%%%%%%%%%%%%%
\begin{align*}
\mathbb{E}_{\xi_{[t]}} \|{\bm x}^{t+1} - {\bm x}^* \| &\leq  (1 + c_\mathcal{T})  
\mathbb{E}_{\xi_{t}|\xi_{[t-1]}}\|{\bm r}^t - 
\eta \nabla_H f_{\xi_t} ({\bm x}^t) \| \\
&\leq \mathcal{\kappa} \mathbb{E}_{\xi_{[t-1]}}\| {\bm x}^t - {\bm x}^* \|
+ \Big(1+c_\mathcal{T}\Big)\Bigg(\frac{\beta_0}{\alpha_0} 
+ \frac{\alpha_0 \beta_0}{\sqrt{1-\alpha_0^2}} 
+ \eta\Bigg) \mathbb{E}_{\xi_{[t]}}  \| \nabla_{I} f_{\xi_t} ({\bm x}^*) \|,
\end{align*}
where $\kappa = (1 + c_\mathcal{T})\Big(\sqrt{ \alpha\beta\eta^2 - 2 \alpha\eta + 1} 
 + \sqrt{1-(\alpha_0)^2}\Big)$. 
We finish the proof by applying the above inequality 
recursively over $t$ iterations:
\begin{equation*}
\mathbb{E}_{\xi_{[t]}} \| {\bm x}^{t+1} - {\bm x}^{*} \|_2 \leq \kappa^{t+1}
\| {\bm x}^0 - {\bm x}^* \|_2 + \frac{1}{1-\kappa} 
\Big(1+c_\mathcal{T}\Big)\Bigg(\frac{\beta_0}{\alpha_0} 
+ \frac{\alpha_0 \beta_0}{\sqrt{1-\alpha_0^2}} 
+ \eta\Bigg) \mathbb{E}_{\xi_{t}}  \| \nabla_{I} f_{\xi_t} ({\bm x}^*) \|.
\end{equation*}
\end{proof}

\addtocounter{corollary}{-3}
\begin{corollary}
Suppose the random distribution is uniform, i.e., $nPr(\xi_t) = 1$ and 
the approximation factor of the head projection can be arbitrary 
boosted close to 1. Taking $\eta = 1 /\beta$, in order 
to get an effective linear convergence rate, i.e., $\kappa < 1$,
the inverse of the condition number $\mu = \alpha/\beta$ and 
tail projection factor $c_\mathcal{T}$ need to satisfy
\begin{equation*}
(1+c_\mathcal{T})(1+2\sqrt{\mu})\sqrt{1- \mu} < 1.
\end{equation*}
\end{corollary}
\begin{proof}
To analyze $\kappa$, we write $\kappa$ as a function of $\eta$, 
i.e., $\kappa(\eta) = (1 + c_\mathcal{T}) (\sqrt{ \alpha\beta\eta^2 
- 2 \alpha \eta + 1}  + \sqrt{1-\alpha_0^2})$. 
We claim that $\kappa(\eta)$ takes its minimum at 
$\eta = 1/\beta$. To see this, we have 
\begin{align*}
\kappa(\eta) &= (1 + c_\mathcal{T}) \Big(\sqrt{\alpha\beta(\eta -1/\beta 
)^2 + 1 - \mu}  + \sqrt{1-\alpha_0^2}\Big) \\
&\geq (1 + c_\mathcal{T}) \Big(\sqrt{ 1 - \mu}  + 
\sqrt{1-\alpha_0^2}\Big).
\end{align*}
Therefore, when $\eta = 1/\beta$, $\kappa(\eta)$ takes 
its minimum. Recall $\alpha_0 = c_\mathcal{H} \alpha \tau - 
\sqrt{\alpha\beta\tau^2 - 2\alpha\tau + 1}$, $\tau \in (0,2/\beta)$. Again, 
we redefine $\alpha_0$ as a function of $\tau$
\begin{equation*}
\alpha_0(\tau) = c_\mathcal{H} \alpha \tau - 
\sqrt{\alpha\beta\tau^2 - 2\alpha\tau + 1}, \tau \in (0, 2/\beta).
\end{equation*}
$\alpha_0(\tau)$ is a concave function with respect to $\tau$. 
To see this, the first and second derivative of $\alpha_0(\tau)$ are
\begin{align*}    
\alpha_0{'}(\tau) &= c_\mathcal{H} \alpha - \frac{\alpha\beta\tau - 
\alpha}{\sqrt{\alpha\beta\tau^2 - 2\alpha\tau + 1}},\quad\quad\quad\quad
 \alpha_0{''}(\tau) =  \Big(\alpha\beta\tau^2 - 2\alpha\tau + 1 
 \Big)^{-3/2} \Big( \alpha^2 - \alpha\beta\Big).
\end{align*}
Hence, $\alpha_0^-{''}(\tau) < 0$ for $\tau \in (0,2/\beta)$. To get the 
maximum of $\alpha_0(\tau)$, let ${\alpha_0}'(\tau) = 0$,
 and solve the following equation
\begin{equation*}
 c_\mathcal{H} \alpha = \frac{\alpha\beta \tau - 
 \alpha}{\sqrt{\alpha\beta\tau^2 - 2\alpha\tau + 1}}.
\end{equation*}
We get two solutions 
\begin{equation*}
\tau_{1} = \frac{1}{\beta} \Bigg(1 + c_\mathcal{H} 
\sqrt{\frac{\beta - \alpha}{\beta - 
c_\mathcal{H}^2\alpha}} \Bigg),\quad\quad\quad\quad 
\tau_{2} = \frac{1}{\beta} \Bigg(1 - c_\mathcal{H} 
\sqrt{\frac{\beta - \alpha}{\beta - c_\mathcal{H}^2\alpha}} \Bigg).
\end{equation*}
Since we assume that the approximation factor $c_\mathcal{H}$ 
of head-oracle can be boosted to any arbitrary constant value 
close to $1$, i.e., $c_\mathcal{H} \rightarrow 1^-$. 
Therefore, 
\begin{equation*}
\lim_{c_\mathcal{H}\rightarrow 1^-} \alpha_0(\tau_1)=2\mu-1,
\lim_{c_\mathcal{H}\rightarrow 1^-} \alpha_0(\tau_2)=-1.
\end{equation*}
As we require $\alpha_0 >0$, we only take $\alpha_0(\tau_1)$. Finally,
the upper bound of $\kappa$ can be simplified as 
\begin{align*}
\lim_{c_\mathcal{H}\rightarrow 1^{-}} \kappa &= 
(1 + c_\mathcal{T}) 
\Bigg( \sqrt{ 1 - \mu} + \sqrt{ 1 - \lim_{c_\mathcal{H}\rightarrow 1^{-}} \alpha_0(\tau_1)^2} \Bigg) \\
&= (1+c_\mathcal{T})\Big(\sqrt{1- \mu} + \sqrt{1-(2\mu - 1)^2 } \Big) \\
&= (1+c_\mathcal{T}) \Big(\sqrt{1- \mu} + 2\sqrt{\mu-\mu^2} \Big) \\
&= (1+c_\mathcal{T}) (1+2\sqrt{\mu}) \sqrt{1- \mu} < 1.
\end{align*}
Hence we prove the corollary.
\end{proof}

\begin{corollary}
If $F({\bm x})$ in ~(\ref{objective:least_square_model}) satisfies the $\alpha$-RSC property and each $f_i ({\bm x})=\frac{n}{2m}\|{\bm A}_{B_i} {\bm x} - {\bm y}_{B_i}\|^2$ satisfies the $\beta$-RSS property, then we have the following condition
\begin{equation}
\frac{\alpha}{2} \|{\bm x}\|^2 \leq \frac{1}{2m}\| 
{\bm A}{\bm x} \|^2 \leq \frac{\beta}{2} \|{\bm x}\|^2. \nonumber
\end{equation}
Let the strong convexity parameter $\alpha = 1-\delta$ and strong smoothness parameter $\beta=1+\delta$, where $0<\delta <1$. The condition of effective contraction factor is as the following
\begin{equation*}
(1+c_\mathcal{T})\Big(\sqrt{\frac{2}{1+\delta}} + \frac{2\sqrt{2(1-\delta)}}{1+\delta}\Big)\sqrt{\delta} < 1.
\end{equation*} To obtain the condition of getting effective contraction factor, we need to have $\delta = \frac{1-\mu}{1+\mu} \leq 7/493 \approx 0.014199$.
\label{corollary:corollary_04}
\end{corollary}
\begin{proof}
The derivative of $f_i({\bm x})$ is
$ \nabla f_i ({\bm x}) = \frac{n}{m} {\bm A}_{B_i}^\mathsf{T} ({\bm A}_{B_i} {\bm x} - {\bm y}_{B_i}) $. $\| \nabla f_i({\bm x}) - \nabla f_i ({\bm y})\|$ can be upper bounded as
\begin{align*}
\| \nabla f_i({\bm x}) - \nabla f_i ({\bm y})\| &= \frac{n}{m} \| {\bm A}_{B_i}^\mathsf{T} {\bm A}_{B_i} ({\bm x} - {\bm y}) \| \leq \beta \| {\bm x} -{\bm y} \|,
\end{align*}
where the inequality follows by $\beta$-RSS of $f_i({\bm x})$, i.e. $\| \nabla f_i({\bm x})  - \nabla f_i({\bm y}) \| \leq \beta\| {\bm x} -{\bm y}\|$. Sum above $n$ inequalities together and divided by $2/n$,  then it is equivalent to say that $F ({\bm x})$ satisfies $\frac{1}{2m} \| {\bm A} {\bm x} \|^2 \leq \frac{\beta}{2}\|{\bm x} \|^2$, where ${\bm x} \in \mathcal{M}(\mathbb{M})$. In order to satisfy $\alpha$-RSC, the Hessian matrix $\frac{1}{m}{\bm A}^\mathsf{T} {\bm A}$ also needs to satisfy
\begin{equation*}
\frac{\alpha}{2} \| {\bm x} \|^2 \leq \frac{1}{2} {\bm x}^\mathsf{T} \frac{{\bm A}^\mathsf{T} {\bm A}}{m} {\bm x}.
\end{equation*}
Combine above inequalities, we prove the condition. Therefore, under the RIP requirement, $F({\bm x})$ is $(1-\delta)$-RSC and $(1+\delta)$-RSS. In order to satisfy Theorem~\ref{theorem:theorem_01}, we require that
\begin{equation}
\kappa = (1+c_T)(1+2\sqrt{\mu})\sqrt{1-\mu} < 1. \nonumber
\end{equation}
By replacing $\mu = \frac{1-\delta}{1+\delta}$, we prove the condition of the effective contraction factor. If $\mu \geq \frac{243}{250}$ and $c_\mathcal{T}$ is enough close to 1, then $\kappa < 1$. Finally, $\delta = \frac{1-\mu}{1+\mu} \leq 7/493 \approx 0.014199$.
\end{proof}

\begin{corollary}
Suppose we have the logistic loss, $F({\bm x})$ in~(\ref{objective:logistic_regression}) and each sample ${\bm a}_i$ is normalized, i.e., $\| {\bm a}_i \| = 1$. Then $F({\bm x})$ satisfies $\lambda$-RSC and each $f_{i} ({\bm x})$ in~(\ref{objective:logistic_regression_i}) satisfies $(\alpha + (1+\nu) \theta_{max})$-RSS. The condition of getting effective contraction factor of \textsc{GraphStoIHT} is
as the following
\begin{align*}
\frac{\lambda}{\lambda + n(1+\nu)\theta_{max} / 4m} \geq \frac{243}{250},
\end{align*}
with probability $1-p \exp{(-\theta_{max}\nu/4)}$, where $\theta_{max} = \lambda_{max} (\sum_{j=1}^{m/n} \mathbb{E}[{\bm a}_{i_j}{\bm a}_{i_j}^\mathsf{T}])$ and $\nu \geq 1$.
\label{corollary:corollary2}
\end{corollary}

\begin{proof}
The derivative and  the Hessian of $f_i({\bm x})$ are
\begin{equation}
\nabla f_i({\bm x}) = \frac{n}{m} \sum_{j=1}^{m/n} \Bigg( - \frac{y_{i_j} {\bm a}_{i_j} }{ 1 + \exp{(y_{i_j} {\bm a}_{i_j}^\mathsf{T} {\bm x}}) } + \lambda {\bm x} \Bigg),
\quad\quad
{\bm H}({\bm x}) = \frac{n}{4m}{\bm A}_{B_i}^\mathsf{T} {\bm Q} {\bm A}_{B_i} + \lambda {\bm I},
\label{corollary:2:inequ1}
\end{equation}
where ${\bm Q}$ is an $b \times b$ diagonal matrix whose diagonal entries are ${\bm Q}_{i_ji_j} = sech^2(\frac{1}{2} y_{i_j} {\bm a}_{i_j}^\mathsf{T}{\bm x})$ and $sech$ is defined as $sech(x) = \frac{2}{\exp{(x)} + \exp{(-x)}}$. To see this, the Hessian of $f_i({\bm x})$ can be calculated as
\begin{align}
{\bm H}({\bm x}) &= \frac{n}{m} \sum_{j=1}^{m/n} \Bigg( \frac{y_{i_j}^2 \exp{(y_{i_j} {\bm a}_{i_j}^\mathsf{T} {\bm x})} {\bm a}_{i_j} {\bm a}_{i_j}^\mathsf{T} }{ (1 + \exp{(y_{i_j} {\bm a}_{i_j}^\mathsf{T} {\bm x})})^2 } + \lambda {\bm I}\Bigg) \nonumber \\
&= \frac{n}{m} \sum_{j=1}^{m/n} \Bigg( \frac{{\bm a}_{i_j} {\bm a}_{i_j}^\mathsf{T} }{ (1 + \exp{(-y_{i_j} {\bm a}_{i_j}^\mathsf{T} {\bm x})})(1 + \exp{(y_{i_j} {\bm a}_{i_j}^\mathsf{T} {\bm x})}) } + \lambda {\bm I}\Bigg) \nonumber \\
&= \frac{n}{m} \sum_{j=1}^{m/n} \Bigg( \frac{{\bm a}_{i_j} {\bm a}_{i_j}^\mathsf{T} }{ (\exp{(\frac{1}{2} y_{i_j} {\bm a}_{i_j}^\mathsf{T} {\bm x})} + \exp{(- \frac{1}{2} y_{i_j} {\bm a}_{i_j}^\mathsf{T} {\bm x})})^2 } + \lambda {\bm I}\Bigg) \nonumber \\
&= \frac{n}{4m} \sum_{j=1}^{m/n} \Bigg( {\bm a}_{i_j} \Big( \frac{2}{ \exp{(\frac{1}{2} y_{i_j} {\bm a}_{i_j}^\mathsf{T} {\bm x})} + \exp{(- \frac{1}{2} y_{i_j} {\bm a}_{i_j}^\mathsf{T} {\bm x})} }\Big)^2 {\bm a}_{i_j}^\mathsf{T}  + 4 \lambda {\bm I}\Bigg) \nonumber \\
&= \frac{n}{4m}{\bm A}_{B_i}^\mathsf{T} {\bm Q} {\bm A}_{B_i} + \lambda {\bm I}.\nonumber
\end{align}
Since ${\bm Q}_{i_ji_j} \leq 1$ and ${\bm A}_{{\bm B}_i}^\mathsf{T} {\bm Q} {\bm A}_{B_i} $ is symmetric matrix, we have
\begin{equation}
\lambda {\bm I} \preceq {\bm H}({\bm x}) \preceq \lambda {\bm I} + \frac{n}{4m} {\bm A}_{B_i}^\mathsf{T} {\bm A}_{{\bm B}_i},
\end{equation}
where $0 \preceq {\bm M}$ means ${\bm M}$ is a positive semi-definite matrix. Hence, $F({\bm x})$ is $\lambda$-RSC. Given the derivative in~(\ref{corollary:2:inequ1}), we aim at getting the upper bound of $\| \nabla f_i ({\bm x}) - \nabla f_i ({\bm y}) \|$ as follows
\begin{align*}
\| \nabla f_i ({\bm x}) - \nabla f_i ({\bm y}) \| &= \Big\| \lambda ({\bm x} - {\bm y}) + \frac{n}{m} \sum_{j=1}^{m/n} y_{i_j} {\bm a}_{i_j} (\frac{1}{1 + \exp {(y_{i_j} {\bm a}_{i_j}^\mathsf{T}{\bm y})}} - \frac{1}{1 + \exp{(y_{i_j} 
{\bm a}_{i_j}^\mathsf{T}{\bm x})}})\Big\|\\
&= \Big\| \lambda ({\bm x} -{\bm y}) - \frac{n}{m} \sum_{j=1}^{m/n} \frac{  y_{i_j}^2   \exp{(y_{i_j} {\bm a}_{i_j}^\mathsf{T} {\bm \theta})} }{ ( 1 + \exp{(y_{i_j} {\bm a}_{i_j}^\mathsf{T}{\bm \theta})} )^2 } {\bm a}_{i_j} {\bm a}_{i_j}^\mathsf{T} ({\bm x} -{\bm y}) \Big\| \\
&= \Big\| \Big( \lambda {\bm I} - \frac{n}{m} \sum_{j=1}^{m/n} \frac{ \exp{(y_{i_j} {\bm a}_{i_j}^\mathsf{T} {\bm \theta})} }{ ( 1 + \exp{(y_{i_j} {\bm a}_{i_j}^\mathsf{T}{\bm \theta})} )^2 } {\bm a}_{i_j} {\bm a}_{i_j}^\mathsf{T} \Big) ({\bm x} -{\bm y}) \Big\| \\
&= \Big\| \Big( \lambda {\bm I} - \frac{n}{4m}{\bm A}_{B_i}^\mathsf{T} {\bm Q} {\bm A}_{B_i} \Big) ({\bm x} -{\bm y}) \Big\|,
\end{align*}
where the first equality is from the derivative of $f_i ({\bm x})$ and $f_i ({\bm y})$ , the second equality follows by using the mean-value theorem for $\phi({\bm x}) = 1/( 1 + \exp{( y_i {\bm a }_i^\mathsf{T} {\bm x})})$, i.e. $\phi ({\bm y}) - \phi ({\bm x}) = \nabla \phi ({\bm \theta}) ({\bm y} - {\bm x})$, ${\bm \theta} \in \{ (1-t) {\bm x} + t{\bm y} : 0 \leq t \leq 1\}$, and the third equality follows by $y_i^2 = 1$. 

We define $\| {\bm M} \|_2$ as a matrix norm, i.e., $\|{\bm M} \|_2 = \sup \{ \|{\bm M}{\bm x}\|/\|{\bm x}\| : {\bm x} \in \mathbb{R}^p \text{ with } {\bm x}\ne {\bm 0} \}$, then we always have $\| {\bm M} {\bm x} \| \leq \| {\bm M} \|_2 \cdot \| {\bm x} \|$. Hence, $\| \nabla f_i ({\bm x}) - \nabla f_i ({\bm y})\|$ can be upper bounded as
\begin{align*}
\| \nabla f_i ({\bm x}) - \nabla f_i ({\bm y}) \|  &\leq \Big\| \lambda {\bm I} - \frac{n}{4m}{\bm A}_{B_i}^\mathsf{T} {\bm Q} {\bm A}_{B_i}  \Big\|_2 \| {\bm x} -{\bm y} \|  \\
&\leq  \Big( \lambda + \frac{n}{4m} \Big\|  {\bm A}_{B_i}^\mathsf{T} {\bm Q} {\bm A}_{B_i} \Big\|_2 \Big) \| {\bm x} -{\bm y} \| \\
&\leq  \Big( \lambda + \frac{n}{4m} \Big\|  {\bm A}_{B_i}^\mathsf{T} {\bm A}_{B_i} \Big\|_2 \Big) \| {\bm x} -{\bm y} \|,
\end{align*}
where the second inequality follows by the triangle inequality of matrix norm, i.e., $\|{\bm A} + {\bm B}\|_2 \leq \|{\bm A} \|_2 + \|{\bm B} \|_2$. Therefore, $f_i ({\bm x})$ satisfies $(\lambda,\lambda+ \frac{n}{4m}\|  {\bm A}_{B_i}^\mathsf{T} {\bm A}_{B_i} \|_2,\mathcal{M}(\mathbb{M}))$-RSS/RSC property. Furthermore, suppose the data is normalized, i.e., $\|{\bm a}_i {\bm a}_i^\mathsf{T} \|_2 = \| {\bm a}_i \|^2 = 1$, by using Lemma~\ref{lemma:lemma6}, we can bound $\|{\bm A}_{B_i}^\mathsf{T} {\bm A}_{B_i}\|_2$ as $\|{\bm A}_{B_i}^\mathsf{T} {\bm A}_{B_i}\|_2 \leq (1+\nu)\theta_{max}$, where $\nu \geq 1$ and $\theta_{\max}$ is defined above. In order to satisfy Theorem~\ref{theorem:theorem_01}, we require
\begin{align*}
\kappa = (1+c_\mathcal{T}) (1+2\sqrt{\mu}) (\sqrt{1- \mu}) < 1,
\end{align*}
where the inverse condition number of logistic regression is $\mu = \frac{\lambda}{\lambda + n(1+\nu)\theta_{max} / 4m}$. After calculation, $\mu \geq \frac{243}{250}$. We prove the corollary.
\end{proof}

\section{Results from benchmark datasets}
\begin{figure}[H]
\centering
{\includegraphics[width=16.0cm,height=8.0cm]{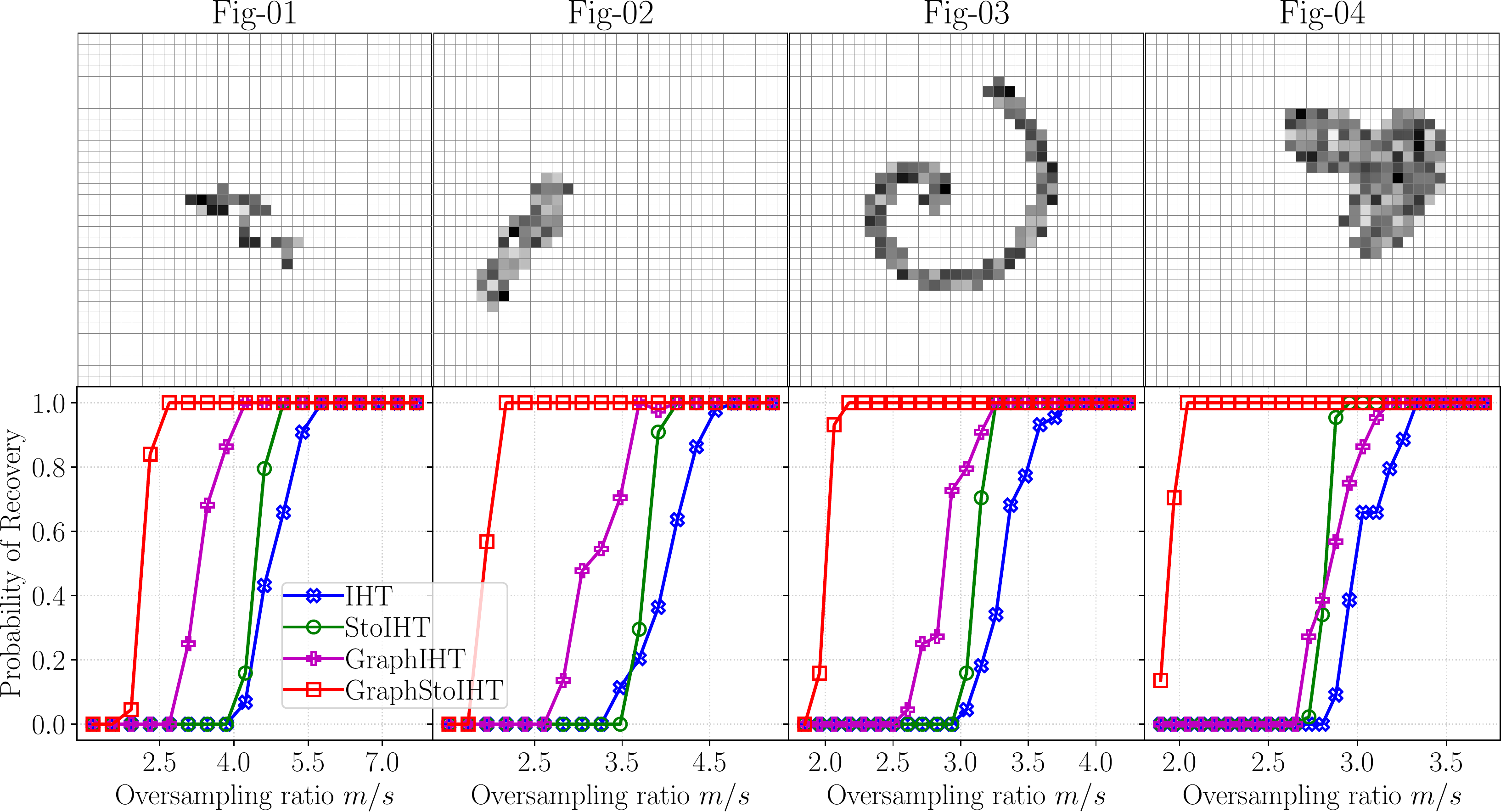}}
\caption{The probability of recovery as a function of oversampling ratio. The oversampling ratio is defined as the number of observations $m$ divided by sparsity $s$, i.e., $m/s$. These four public benchmark graphs (a), (b), (c), and (d)  in the upper row are from~\citet{arias2011detection}.}
\label{fig:simu_figs_re_05}
\end{figure}
\textbf{Implementation details of \textsc{GraphStoIHT}.} The head and tail projection of \textsc{GraphStoIHT} are implemented in C language. We exactly follow the original implementation\footnote{The original implementation is available at: \url{https://github.com/ludwigschmidt/cluster_approx}.}. The two projections are essentially two binary search algorithms. Each iteration of the binary search executes the Prize Collecting Steiner Tree (PCST) algorithm~\cite{johnson2000prize} on the target graph. Both projections have two main parameters: a lower bound sparsity $s_l$ and an upper bound  sparsity $s_u$. In all of the experiments, two sparsity parameters have been set to $s_l = p/2$ and $s_u=s_l*(1 + \omega)$ for the head projection, where the $\omega$ is the tolerance parameter set to $0.1$. For the tail projection, we set $s_l = s$ and $s_u = s_l*(1 + \omega)$. The binary search algorithm terminates when it reaches 50 maximum iterations.

In order to demonstrate \textsc{GraphStoIHT} can handle different shapes of graph structure, we consider four public benchmark graphs in~\citet{arias2011detection}. In the upper row of Figure~\ref{fig:simu_figs_re_05}, all of the subgraphs are embedded in $33\times 33$ grid graph. The sparsity of these four graphs is 26, 46, 92, and 132 respectively. All settings keep consistent with Section~\ref{section:section_5.1}. The learning rate $\eta$ and block size are tuned from $\{0.2, 0.4, 0.6, 0.8\}$ and $\{m/5, m/10\}$ on additional $100$ observations. The results are reported in the bottom row of Figure~\ref{fig:simu_figs_re_05}. It shows that when the sparsity $s$ increases, the number of observations required is also increasing. This is
consistent with the lower bound of the number of observations $\mathcal{O}(s\log(p/s))$, which increases as $s$ increases. Our method is consistent with the bound $\mathcal{O}(s\log(d(v_i)) + \log p)$ shown in Equation (5) of ~\citet{hegde2015nearly} where the weight-degree $d(v_i)$ is 4 and the budget is $s-1$. To further compare our method with \textsc{GraphCoSaMP}~\cite{hegde2015nearly}, we repeat the experiment conducted in Appendix A of~\citet{hegde2015nearly} and report the results in Figure~\ref{fig:simu_figs_re_06}. Thanks to the introduced randomness, \textsc{GraphStoIHT} outperforms \textsc{GraphCoSaMP}.

\begin{figure}[ht!]
\centering
{\includegraphics[width=14cm,height=6cm]{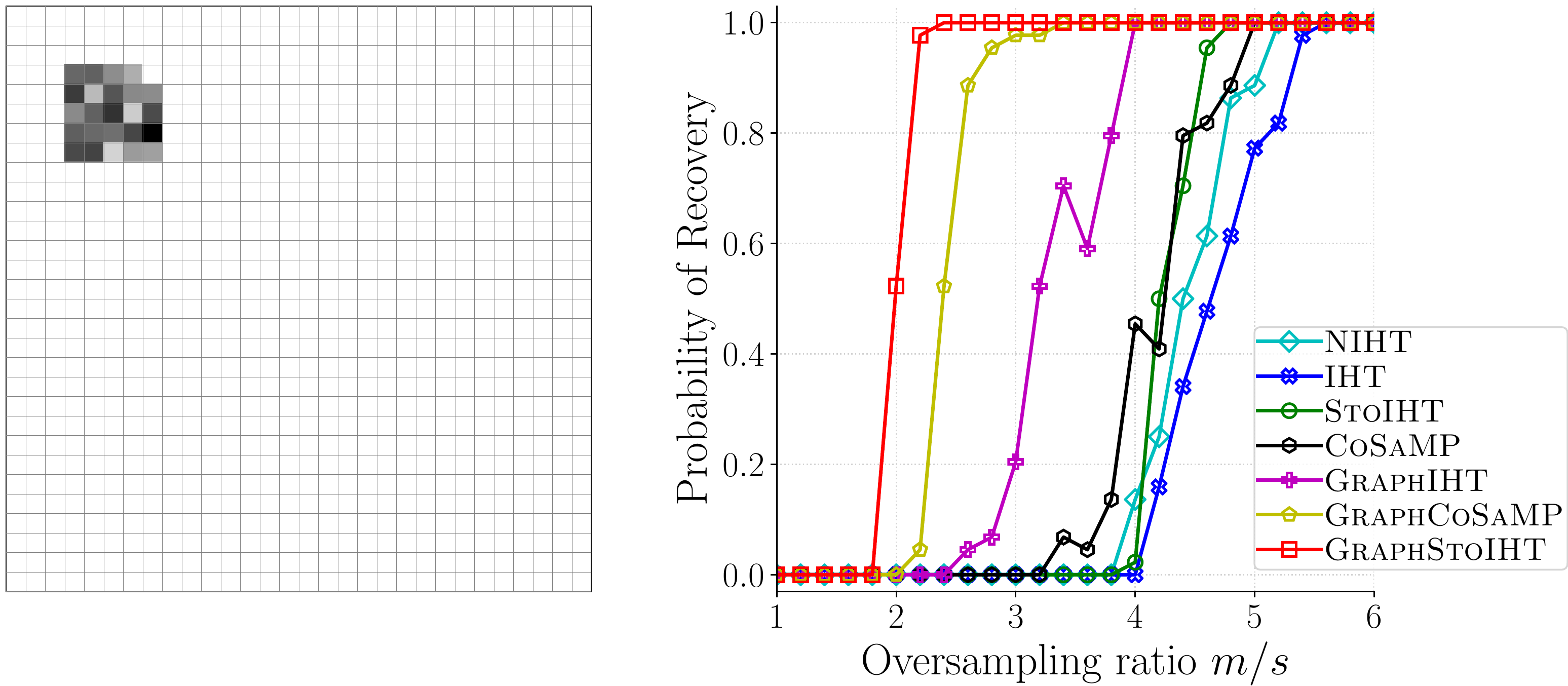}}
\caption{The probability of recovery as a function of oversampling ratio. This grid graph is from~\citet{hegde2015nearly}.}
\label{fig:simu_figs_re_06}
\end{figure}

\section{Results from breast cancer dataset}
\textbf{Parameter tuning.\quad} All parameters are selected by using 5-fold-cross-validation. The sparsity $s$ is tuned from $\{10, 15, 20, 25,\ldots, 90, 95, 100\}$. The regularization parameter $\lambda$ is tuned from $\{10^{-3}, 10^{-4}\}$. The block size $b$ is tuned from $\{m, m/2\}$. Since we need to find a connected subgraph, the number of connected component is $g=1$. The 4 non-convex methods all use backtracking line search (Armijo–Goldstein condition) to obtain the learning rate at each iteration and stop at the maximum iteration of 40. The dataset has been normalized by using $z$-score, i.e., each sample ${\bm a}_i$ is subtracted by mean and divided by the standard deviation of $\{{\bm a}_i\}_{i=1}^{295}$.

We first measure the performance by using the balanced classification error in~\citet{jacob2009group}. The results are reported in Table~\ref{table:table_03}. \textsc{GrapIHT} produces the lowest balanced classification error while the results of \textsc{GraphStoIHT} are still competitive. We argue that this may be caused by insufficient samples.

\textbf{Gene identification.\quad} We pooled 25 genes from~\citet{phaedra2009bayesian, couch2017associations,rheinbay2017recurrent, gyorffy2010online}. These genes are highly related with breast cancer, namely, MKI67, CDKN1A, ATM, TFF3, FBP1, XBP1, DSC2, CDH3, CHEK2, TOP2A, BRCA1, BRCA2, BARD1, NAT1, CA12, AR, TK2, RAD51D, GATA3, TOP2B, CCND3, CCND2, AGR2, FOXA1, and FOXC1. \textsc{GraphStoIHT} identifies more breast cancer-related genes but use less gene features (see Table~\ref{table:table_04}) than $\ell^1/\ell^2$-\textsc{Pathway} and \textsc{GraphIHT}.

\begin{table*}[ht!]
\caption{Breast cancer-related genes identified by different algorithms.}
\centering
\begin{tabular}{ccc}
\hline
\hline
Algorithm & Number of genes & Cancer related genes \\
\hline
\textsc{GraphStoIHT} & 6 & BRCA2, CCND2, CDKN1A, ATM, AR, TOP2A \\
\textsc{GraphIHT} & 5 & ATM, CDKN1A, BRCA2, AR, TOP2A \\
$\ell^1$-\textsc{Pathway} & 4 & BRCA1, CDKN1A, ATM, DSC2 \\
\textsc{StoIHT} & 4 & MKI67, NAT1, AR, TOP2A \\
\textsc{StoIHT} & 4 & MKI67, NAT1, AR, TOP2A \\
$\ell^1/\ell^2$-\textsc{Edge} & 3 & CCND3, ATM, CDH3 \\
$\ell^1$-\textsc{Edge} & 3 & CCND3, AR, CDH3 \\
$\ell^1/\ell^2$-\textsc{Pathway} & 2 & BRCA1, CDKN1A \\
\textsc{IHT} & 2 & NAT1, TOP2A \\
\hline
\end{tabular}
\label{table:table_05}
\end{table*}

\begin{table*}[ht!]
\caption{Balanced Classification Error on the breast cancer dataset.}
\centering
\scriptsize
\begin{tabular}{ccccccccccc}
\hline
Folding ID & $\ell^1$-\textsc{Pathway} & $\ell^1/\ell^2$-\textsc{Pathway} & $\ell^1$-\textsc{Edge} & $\ell^1/\ell^2$-\textsc{Edge} & \textsc{IHT} & \textsc{StoIHT} & \textsc{GraphIHT} & \textsc{GraphStoIHT} \\
\hline
Folding 00 & 0.349$\pm$0.06 & 0.349$\pm$0.07 & 0.339$\pm$0.04 & 0.358$\pm$0.06 & 0.352$\pm$0.08 & 0.350$\pm$0.08 & \textbf{0.322}$\pm$0.04 & 0.380$\pm$0.05 \\
Folding 01 & 0.419$\pm$0.09 & \textbf{0.321}$\pm$0.04 & 0.341$\pm$0.05 & 0.364$\pm$0.06 & 0.369$\pm$0.08 & 0.359$\pm$0.08 & 0.358$\pm$0.08 & 0.383$\pm$0.07 \\
Folding 02 & 0.398$\pm$0.07 & 0.314$\pm$0.05 & 0.340$\pm$0.05 & 0.315$\pm$0.05 & 0.353$\pm$0.05 & 0.348$\pm$0.07 & \textbf{0.311}$\pm$0.06 & 0.322$\pm$0.06 \\
Folding 03 & 0.396$\pm$0.03 & 0.342$\pm$0.04 & \textbf{0.336}$\pm$0.07 & 0.352$\pm$0.07 & 0.358$\pm$0.06 & 0.342$\pm$0.06 & 0.365$\pm$0.03 & 0.365$\pm$0.03 \\
Folding 04 & 0.380$\pm$0.06 & 0.366$\pm$0.06 & 0.356$\pm$0.08 & 0.358$\pm$0.05 & \textbf{0.349}$\pm$0.08 & 0.368$\pm$0.08 & 0.357$\pm$0.06 & 0.350$\pm$0.05 \\
Folding 05 & 0.418$\pm$0.04 & 0.330$\pm$0.07 & 0.386$\pm$0.11 & 0.368$\pm$0.08 & 0.353$\pm$0.04 & 0.347$\pm$0.05 & 0.354$\pm$0.09 & \textbf{0.310}$\pm$0.06 \\
Folding 06 & 0.376$\pm$0.03 & \textbf{0.307}$\pm$0.04 & 0.384$\pm$0.07 & 0.353$\pm$0.04 & 0.343$\pm$0.05 & 0.367$\pm$0.08 & 0.321$\pm$0.03 & 0.321$\pm$0.03 \\
Folding 07 & 0.394$\pm$0.04 & 0.363$\pm$0.08 & 0.360$\pm$0.06 & 0.357$\pm$0.05 & 0.352$\pm$0.06 & 0.334$\pm$0.05 & \textbf{0.327}$\pm$0.04 & 0.363$\pm$0.09 \\
Folding 08 & 0.408$\pm$0.05 & 0.336$\pm$0.05 & \textbf{0.313}$\pm$0.04 & 0.366$\pm$0.07 & 0.340$\pm$0.07 & 0.340$\pm$0.07 & 0.355$\pm$0.05 & 0.357$\pm$0.05 \\
Folding 09 & 0.366$\pm$0.05 & 0.374$\pm$0.05 & 0.394$\pm$0.03 & 0.359$\pm$0.03 & 0.343$\pm$0.07 & 0.343$\pm$0.07 & \textbf{0.335}$\pm$0.06 & 0.335$\pm$0.05 \\
Folding 10 & 0.353$\pm$0.07 & \textbf{0.313}$\pm$0.08 & 0.362$\pm$0.08 & 0.350$\pm$0.05 & 0.358$\pm$0.07 & 0.334$\pm$0.07 & 0.356$\pm$0.05 & 0.340$\pm$0.07 \\
Folding 11 & 0.410$\pm$0.06 & 0.376$\pm$0.05 & 0.370$\pm$0.08 & 0.342$\pm$0.06 & 0.369$\pm$0.05 & 0.384$\pm$0.07 & \textbf{0.329}$\pm$0.08 & 0.329$\pm$0.08 \\
Folding 12 & 0.371$\pm$0.04 & 0.355$\pm$0.07 & 0.371$\pm$0.07 & 0.377$\pm$0.04 & 0.327$\pm$0.05 & 0.331$\pm$0.07 & \textbf{0.310}$\pm$0.06 & 0.357$\pm$0.08 \\
Folding 13 & 0.405$\pm$0.06 & 0.385$\pm$0.04 & 0.383$\pm$0.04 & 0.384$\pm$0.05 & 0.362$\pm$0.06 & 0.386$\pm$0.08 & 0.355$\pm$0.04 & \textbf{0.325}$\pm$0.05 \\
Folding 14 & 0.417$\pm$0.04 & 0.358$\pm$0.05 & 0.374$\pm$0.03 & 0.354$\pm$0.06 & 0.360$\pm$0.05 & 0.346$\pm$0.06 & \textbf{0.297}$\pm$0.06 & 0.297$\pm$0.06 \\
Folding 15 & 0.416$\pm$0.08 & 0.381$\pm$0.04 & 0.402$\pm$0.07 & 0.398$\pm$0.07 & 0.332$\pm$0.09 & 0.349$\pm$0.08 & \textbf{0.329}$\pm$0.07 & 0.331$\pm$0.06 \\
Folding 16 & 0.382$\pm$0.07 & 0.357$\pm$0.07 & 0.363$\pm$0.05 & 0.345$\pm$0.07 & 0.341$\pm$0.06 & 0.334$\pm$0.07 & \textbf{0.322}$\pm$0.06 & 0.331$\pm$0.06 \\
Folding 17 & 0.385$\pm$0.08 & 0.344$\pm$0.09 & 0.347$\pm$0.08 & 0.355$\pm$0.04 & 0.331$\pm$0.10 & 0.331$\pm$0.10 & \textbf{0.290}$\pm$0.07 & 0.327$\pm$0.09 \\
Folding 18 & 0.401$\pm$0.04 & 0.309$\pm$0.06 & 0.382$\pm$0.04 & 0.372$\pm$0.02 & 0.341$\pm$0.05 & 0.394$\pm$0.08 & \textbf{0.302}$\pm$0.03 & 0.318$\pm$0.03 \\
Folding 19 & 0.391$\pm$0.03 & \textbf{0.341}$\pm$0.07 & 0.377$\pm$0.03 & 0.367$\pm$0.06 & 0.352$\pm$0.04 & 0.360$\pm$0.03 & 0.355$\pm$0.06 & 0.343$\pm$0.04 \\
\hline
Averaged  & 0.392$\pm$0.06 & 0.346$\pm$0.06 & 0.364$\pm$0.07 & 0.360$\pm$0.06 & 0.349$\pm$0.07 & 0.352$\pm$0.07 & \textbf{0.332}$\pm$0.06 & 0.339$\pm$0.07 \\
\hline
\end{tabular}
\label{table:table_03}
\end{table*}

\begin{table*}[ht!]
\caption{Number of nonzeros on the breast cancer dataset.}
\centering
\scriptsize
\begin{tabular}{ccccccccccc}
\hline
Folding ID & $\ell^1$-\textsc{Pathway} & $\ell^1/\ell^2$-\textsc{Pathway} & $\ell^1$-\textsc{Edge} & $\ell^1/\ell^2$-\textsc{Edge} & \textsc{IHT} & \textsc{StoIHT} & \textsc{GraphIHT} & \textsc{GraphStoIHT} \\
\hline
Folding 00 & 075.4$\pm$12.92 & 208.6$\pm$71.37 & 053.2$\pm$13.76 & 081.6$\pm$28.88 & 082.0$\pm$36.00 & 088.0$\pm$24.00 & 028.0$\pm$04.00 & 037.6$\pm$10.97 \\
Folding 01 & 081.6$\pm$04.13 & 113.0$\pm$27.62 & 003.8$\pm$02.14 & 059.8$\pm$39.75 & 051.0$\pm$22.00 & 057.0$\pm$22.27 & 095.4$\pm$00.80 & 066.2$\pm$36.55 \\
Folding 02 & 072.2$\pm$10.57 & 108.8$\pm$45.66 & 044.6$\pm$22.54 & 054.2$\pm$19.77 & 075.0$\pm$00.00 & 040.0$\pm$18.44 & 040.6$\pm$00.49 & 024.6$\pm$08.69 \\
Folding 03 & 068.0$\pm$21.38 & 117.8$\pm$54.22 & 041.2$\pm$12.58 & 063.6$\pm$21.84 & 083.0$\pm$06.00 & 040.0$\pm$08.94 & 051.6$\pm$08.33 & 051.6$\pm$08.33 \\
Folding 04 & 044.4$\pm$17.97 & 135.8$\pm$68.41 & 046.6$\pm$14.68 & 049.2$\pm$17.35 & 035.0$\pm$00.00 & 043.0$\pm$17.20 & 028.8$\pm$17.60 & 065.0$\pm$11.33 \\
Folding 05 & 075.6$\pm$20.37 & 121.0$\pm$50.32 & 051.6$\pm$12.97 & 069.0$\pm$15.95 & 085.0$\pm$00.00 & 071.0$\pm$28.00 & 049.4$\pm$12.21 & 032.2$\pm$12.06 \\
Folding 06 & 058.2$\pm$10.15 & 147.6$\pm$47.86 & 042.2$\pm$13.14 & 016.2$\pm$22.45 & 010.0$\pm$00.00 & 044.0$\pm$16.85 & 079.6$\pm$09.20 & 079.8$\pm$09.60 \\
Folding 07 & 040.2$\pm$14.55 & 115.8$\pm$70.68 & 022.8$\pm$09.83 & 050.6$\pm$11.99 & 040.0$\pm$00.00 & 044.0$\pm$03.74 & 098.4$\pm$04.27 & 080.4$\pm$35.45 \\
Folding 08 & 049.0$\pm$23.73 & 137.4$\pm$62.11 & 009.2$\pm$02.48 & 020.4$\pm$07.26 & 095.0$\pm$00.00 & 095.0$\pm$00.00 & 032.2$\pm$32.90 & 040.2$\pm$33.17 \\
Folding 09 & 052.0$\pm$29.54 & 135.8$\pm$34.07 & 024.8$\pm$15.84 & 057.2$\pm$26.11 & 049.0$\pm$02.00 & 049.0$\pm$02.00 & 052.0$\pm$14.00 & 047.2$\pm$04.40 \\
Folding 10 & 052.8$\pm$25.07 & 137.0$\pm$67.58 & 038.6$\pm$17.35 & 027.2$\pm$13.53 & 095.0$\pm$00.00 & 072.0$\pm$30.27 & 083.0$\pm$05.06 & 074.8$\pm$18.09 \\
Folding 11 & 063.8$\pm$14.69 & 202.2$\pm$40.56 & 055.4$\pm$13.88 & 043.8$\pm$32.08 & 044.0$\pm$02.00 & 037.0$\pm$09.27 & 098.4$\pm$09.60 & 098.4$\pm$09.60 \\
Folding 12 & 057.2$\pm$23.96 & 217.4$\pm$28.99 & 040.8$\pm$15.59 & 059.2$\pm$06.08 & 046.0$\pm$12.00 & 036.0$\pm$17.72 & 068.2$\pm$24.10 & 041.2$\pm$16.67 \\
Folding 13 & 068.0$\pm$18.44 & 177.4$\pm$91.56 & 056.6$\pm$34.55 & 073.8$\pm$22.13 & 077.0$\pm$26.00 & 061.0$\pm$12.00 & 088.6$\pm$23.80 & 031.0$\pm$03.74 \\
Folding 14 & 066.6$\pm$19.89 & 148.4$\pm$27.38 & 038.6$\pm$12.56 & 040.8$\pm$15.52 & 070.0$\pm$10.00 & 066.0$\pm$14.28 & 045.0$\pm$00.00 & 045.0$\pm$00.00 \\
Folding 15 & 062.8$\pm$09.99 & 215.2$\pm$51.84 & 037.0$\pm$08.41 & 063.0$\pm$14.21 & 056.0$\pm$18.00 & 056.0$\pm$22.00 & 029.4$\pm$02.33 & 063.0$\pm$28.57 \\
Folding 16 & 051.2$\pm$23.59 & 075.2$\pm$39.45 & 045.8$\pm$28.60 & 060.8$\pm$16.09 & 088.0$\pm$06.00 & 077.0$\pm$16.91 & 029.6$\pm$01.85 & 029.2$\pm$02.14 \\
Folding 17 & 050.6$\pm$15.53 & 147.8$\pm$18.14 & 043.8$\pm$32.50 & 039.4$\pm$23.10 & 094.0$\pm$12.00 & 094.0$\pm$12.00 & 034.2$\pm$02.14 & 029.2$\pm$09.83 \\
Folding 18 & 069.6$\pm$11.45 & 116.8$\pm$53.54 & 037.8$\pm$22.43 & 044.2$\pm$21.33 & 088.0$\pm$06.00 & 041.0$\pm$13.19 & 038.6$\pm$03.83 & 050.6$\pm$13.65 \\
Folding 19 & 065.0$\pm$30.55 & 172.4$\pm$20.71 & 064.4$\pm$04.45 & 053.0$\pm$29.57 & 095.0$\pm$00.00 & 082.0$\pm$11.66 & 032.4$\pm$23.80 & 046.0$\pm$10.20 \\
\hline
Averaged  & 061.2$\pm$22.12 & 147.6$\pm$64.84 & 039.9$\pm$23.09 & 051.4$\pm$27.34 & 067.9$\pm$27.28 & 059.6$\pm$25.70 & 055.2$\pm$28.88 & 051.7$\pm$26.65 \\
\hline
\end{tabular}
\label{table:table_04}
\end{table*}

\end{document}